%% file: neurips_main.tex
\documentclass{article}

\usepackage[preprint]{neurips_2026}

\usepackage[utf8]{inputenc}
\usepackage[T1]{fontenc}
\usepackage{microtype}
\usepackage{graphicx}
\usepackage{subcaption}
\usepackage{booktabs}
\usepackage{xurl}
\usepackage{hyperref}
\usepackage{url}
\usepackage{amsmath}
\usepackage{amsfonts}
\usepackage{amssymb}
\usepackage{mathtools}
\usepackage{amsthm}
\usepackage{xcolor}
\usepackage{nicefrac}
\usepackage{algorithm}
\usepackage{algorithmic}
\usepackage[capitalize,noabbrev]{cleveref}
\usepackage[normalem]{ulem}
\usepackage[markup=underlined]{changes}
\usepackage{tabularx}
\usepackage{wrapfig}
\usepackage{multirow}

\newcommand{\kaan}[1]{\textcolor{purple}{[Kaan: #1]}}

\newcommand{\E}[1]{\mathbb{E}\left[#1\right]}

\newcommand{\inner}[2]{\left\langle #1, #2 \right\rangle}
\newcommand{\mA}{\mathcal{A}}
 
\newcommand{\abs}[1]{\left|#1 \right|}

\theoremstyle{plain}
\newtheorem{theorem}{Theorem}[section]
\newtheorem{proposition}[theorem]{Proposition}
\newtheorem{lemma}[theorem]{Lemma}
\newtheorem{corollary}[theorem]{Corollary}
\theoremstyle{definition}
\newtheorem{definition}[theorem]{Definition}

\theoremstyle{remark}

\title{Mixing Makes Markovian Contexts Cheap \\for Linear Bandits}

\author{%
  Kaan Buyukkalayci \\
  University of California, Los Angeles \\
  \And
  Osama Hanna \\
  Meta, Superintelligence Lab \\
  \AND
  Christina Fragouli \\
  University of California, Los Angeles \\
}

\begin{document}

\maketitle

\begin{abstract}
Recent work shows that when contexts are drawn i.i.d., linear contextual bandits can be reduced to single-context linear bandits. This ``contexts are cheap'' perspective is highly advantageous, as it allows  for sharper finite-time analyses and leverages mature techniques from the linear bandit literature, such as those for misspecification and adversarial corruption. However, this reduction crucially relies on the independence of contexts and does not extend to settings with temporally correlated (e.g., Markovian) contexts, which arise frequently in practice. Motivated by applications with temporally correlated availability, we extend this perspective to linear bandits with Markovian context processes, where the action set evolves via an exogenous Markov chain. Our main contribution is a reduction that applies under uniform geometric ergodicity. We construct a stationary surrogate action set to solve the problem using a standard linear bandit oracle, employing a delayed-update scheme to control the bias induced by the nonstationary conditional context distributions. We further provide a phased algorithm for unknown stationary distributions that learns the surrogate mapping online. In both settings, we obtain a high-probability worst-case regret bound matching that of the underlying linear bandit oracle in sufficiently fast mixing regimes. We then validate our results on a real-world instance, where we show practical gains over a LinUCB baseline. 
\end{abstract}

\section{Introduction}

Contextual linear bandits are a central framework in online learning, capturing a wide range of sequential decision-making problems where side information guides action selection. Their modeling versatility has enabled extensive real-world adoption, driving advances in applications such as autonomous systems,
personalized online recommendations, online controlled experiments, and diagnostic tools in healthcare
\citep{chacun2024dronebandit,wakayama2023observation, medical,ang2023evaluating,bojinov2022online}.
In many applications, however, the available actions (contexts) evolve over time in a structured manner, rather than being drawn independently across rounds. While recent work has shown that in the i.i.d. setting, ``contexts are cheap''-contextual linear bandits can be reduced to standard linear bandits leading to sharper analyses and improved guarantees—this reduction fundamentally relies on independence and breaks down in the presence of temporal dependence.

In this work, we show that the “contexts are cheap” perspective extends beyond the i.i.d. setting to Markovian context processes. Such temporal dependence arises naturally in applications  where the available actions evolve according to underlying dynamics: the current location and sensor readings of an autonomous robot may depend on its previous states, user behavior in recommendation and online controlled experimentation systems may exhibit temporally correlated interests and population drift, and the progression of a cancer cell may evolve over time. 
The key challenge is that correlations across rounds introduce bias that invalidates the standard reduction-based analysis. To address this, we develop a reduction that combines a stationary surrogate action set with a delayed-feedback mechanism, allowing the context process to mix and thereby controlling the induced bias. Our results show that contexts remain ‘cheap’ even under temporal dependence, provided the process mixes sufficiently fast.

We establish strong regret guarantees for this approach in both known and unknown stationary distribution settings. When the stationary distribution is known, our algorithm achieves high-probability regret that matches that of classical linear bandits in sufficiently fast-mixing regimes. When the stationary distribution is unknown, we introduce a phased algorithm that learns the surrogate mapping online and achieves regret bounds that match linear bandit rates up to mixing-dependent factors. These results demonstrate that temporal dependence in the context process does not fundamentally degrade regret guarantees beyond the cost of mixing.

Beyond the base setting, our framework enables improved guarantees in more challenging scenarios such as misspecification and adversarial corruption, by leveraging existing advances in linear bandit methods. We further validate our approach empirically on a real-world dataset that we plan to release, demonstrating consistent gains over standard contextual bandit baselines such as LinUCB.

\section{Related Work}
\label{sec:related}

The ordinary (single-context) linear bandit problem admits worst-case regret of order $O(d\sqrt{T\log T})$ (e.g., \citet{lattimore2020}), achieved up to logarithmic factors by optimism-based algorithms such as OFUL and known to be minimax-optimal up to logarithmic factors \citep{dani2008stochastic,abbasi2011improved,lattimore2020bandit}. In linear contextual bandits, each round reveals a context-dependent action set, and rewards are linear in an unknown parameter; canonical algorithms such as LinUCB/OFUL and linear Thompson sampling achieve $\tilde O(d\sqrt{T})$ regret under sub-Gaussian noise \citep{chu2011contextual,abbasi2011improved,agrawal2013thompson,lattimore2020bandit}.
While these rates match those of linear bandits in order, reduction-based approaches from stochastic contextual models to linear bandits are particularly powerful: they yield sharper finite-time guarantees, simplify both algorithms and analysis, and enable extensions to settings such as misspecification, batched interaction, and adversarial corruption \citep{hanna2023contexts}.
Our work extends this line by removing the independence assumption and handling temporally correlated (Markovian) context processes.

A related line of work studies stochastic decision sets (i.e., sleeping or stochastically available actions) under adversarial losses, beginning with online combinatorial optimization with stochastic decision sets \citep{neu2014online}. In adversarial linear contextual bandits with randomly varying action sets, recent work has focused on bypassing access to a context simulator while retaining near-optimal regret and computational efficiency \citep{liu2023bypassing,olkhovskaya2023firstsecond}. Most recently, \citet{vanerven2025improved} extended the reduction framework of \citet{hanna2023contexts} to adversarial settings, reducing problems with stochastic action sets to misspecification-robust adversarial linear bandits with fixed actions and achieving $\mathrm{poly}(d)\sqrt{T}$-type regret without knowing the context distribution. 
While our setting assumes stochastic contexts, our results reinforce the theme that randomness in action availability can be exploited via carefully constructed surrogates to obtain better bounds. Our contribution shows that this principle continues to hold even under Markovian dependence across rounds.

A separate line of work studies bandits and reinforcement learning with Markovian structure, including restless bandits and linear MDPs \citep{tekin2012online, ortner2012regret, jin2020provably}. These models typically require learning or planning over state transitions and incur regret bounds that scale with transition complexity. In contrast, we consider a setting where the context process evolves exogenously and rewards remain linear, enabling a reduction to a standard linear bandit rather than a full MDP formulation. Our work shows that 
temporal dependence can be controlled via delayed feedback, allowing regret guarantees that match those of classical linear bandits up to mixing-dependent factors.

Our setting can also be interpreted as a linear Markov decision process with
\emph{exogenous actions}: the context process $(\mathcal A_t)$ evolves autonomously
according to a Markov kernel $P$, independent of the learner’s actions, while rewards
are linear in the selected action features. In principle, this allows the use of
algorithms for linear MDPs, such as LSVI--UCB~\citep{jin2020provably}. However, generic MDP regret bounds are poorly
suited to this regime, as they incur unnecessary pessimism by accounting for transition
uncertainty that is irrelevant when the dynamics are exogenous. By avoiding a full MDP
treatment and instead reducing the problem to a linear bandit with a carefully controlled
bias via delayed feedback, our analysis recovers regret guarantees that match the best
known rates for linear bandits whenever the context process mixes sufficiently fast.
An extended version of related work is deferred to Appendix~\ref{app:related}.

\section{Setup and Notation}

\label{sec:not}

We use the shorthand notation $[i] = \{1,2,\dots,i\}$ for any $i \in \mathbb{N}$ with $i > 0$, where $\mathbb{N}$ denotes the set of natural numbers. We write $y = O(f(x))$ if there exist constants $c > 0$ and $x_0 \in \mathbb{R}$ such that $y \leq c f(x), \quad \forall x > x_0$. We use the notation $\tilde{O}(f(x))$ to suppress logarithmic factors. 
 We use $\|\cdot\|_2$ to denote the Euclidean norm, and $\mathrm{TV}(\cdot,\cdot)$ to denote the total variation distance, which is defined by
\[
\mathrm{TV}(P,Q) := \sup_{A \in \mathcal{F}} |P(A) - Q(A)|
= \frac{1}{2} \int |dP - dQ|.
\]

where $P$ and $Q$ are probability measures on a common measurable space $(\Omega,\mathcal{F})$. A \emph{$\delta$-net} of a set $\mathcal{A} \subseteq \mathbb{R}^d$ (with respect to the $\ell_2$ norm) is a set $\mathcal{B}$ such that for every $a \in \mathcal{A}$ there exists $b \in \mathcal{B}$ with $\|a - b\|_2 \leq \delta$ where $\delta > 0$.

At each round $t \in [T]$, the learner plays an action $a_t$ and observes a reward,
$$
r_t = \langle a_t, \theta_\star \rangle + \eta_t
$$
where the action $a_t \in \mathcal{A}_t$ is chosen from the available action set
$\mathcal{A}_t$, which is termed the \emph{context}. $\theta_\star \in \Theta \subseteq \mathbb{R}^d$
is an unknown parameter, and $\eta_t$ is sub-Gaussian noise satisfying
$\mathbb{E}[\eta_t \mid \mathcal{F}_t] = 0$ and
$\mathbb{E}[\exp(\lambda \eta_t) \mid \mathcal{F}_t] \leq \exp(\lambda^2/2)$
for all $\lambda \in \mathbb{R}$. Here,
$\mathcal{F}_t = \sigma\{\mathcal{A}_1,a_1,r_1,\dots,\mathcal{A}_t,a_t\}$ denotes the filtration
representing the history up to time $t$, with $\sigma(X)$ being the $\sigma$-algebra generated by $X$. We assume each context set $\mathcal A_t$ is compact. Whenever the maximizer $\arg\max \inner{a}{\theta}$ is not
unique, ties are broken according to a fixed measurable rule, for any $a\in\mathcal{A}_t$, $\theta\in\Theta$.
At every round $t$, the context set $\mathcal{A}_t$ is revealed to the learner before selecting an action. 
We adopt the standard boundedness assumptions $\|a\|_2 \leq 1$ for all $a \in \mathcal{A}_t$ 
and $\|\theta\|_2 \leq 1$ for all $\theta \in \Theta$ almost surely. The goal of the learner is to minimize regret defined as 
\[R_T = \sum_{t=1}^T \max_{a \in \mathcal{A}_t}\inner{a}{\theta_\star}-\inner{a_t}{\theta_\star}\]

In addition, we consider the setting in which the sequence of context sets
\((\mathcal A_t)_{t=1}^T\) evolves as a Markov chain on a measurable state space
\((\mathsf S,\mathcal F_{\mathsf S})\), where \(\mathsf S\) is a Polish space and
the dynamics are governed by a transition kernel
\(P:\mathsf S\times\mathcal F_{\mathsf S}\to[0,1]\). For any probability measure $\nu$ on $\mathsf{S}$, we define its pushforward measure under $P$ by
\[
(P\nu)(B) \triangleq \int_{\mathsf{S}} P(x,B)\, d\nu(x), \qquad \forall B \in \mathcal{F_\mathsf{S}},
\]
where $P(x,B)$ denotes the probability that $\mathcal{A}_{t+1} \in B$ given $\mathcal{A}_t = x$, and denote by $P^t \nu$ the $t$-step pushforward measure obtained by $t$ successive applications of $P$. The Markov chain is assumed to be uniformly geometrically ergodic, with unique stationary distribution $\pi$.  We formally state this classical assumption as follows.
\begin{definition}[Uniform Geometric Ergodicity]
The context process $(\mathcal{A}_t)_{t \ge 1}$ is uniformly geometrically ergodic if there exist constants $C_\text{mix} < \infty$ and $\beta \in (0,1)$, depending only on the transition kernel $P$, such that for any initial distribution $\mu$ on $\mathsf{S}$ and all $t \ge 0$:
\begin{equation}
\label{eq:UGE}
\mathrm{TV}(P^t\mu, \pi) \;\le\; C_\text{mix}\,\beta^t,
\end{equation}
where $\pi$ is the unique stationary distribution.
\end{definition}

While our results apply to general (possibly infinite) state spaces, in the
finite-state case irreducibility and aperiodicity already imply a unique stationary
distribution and uniform geometric ergodicity, and therefore this assumption
does not impose an additional restriction. We also note that although the total variation distance in~\eqref{eq:UGE} vanishes when the chain is initialized from the stationary distribution, our bounds hold uniformly over all initial distributions. In particular, the analysis reveals no substantive advantage to stationary initialization.






\section{Known Stationary Distribution}

\label{sec:known}

In this section we present our algorithm for the case where the stationary distribution $\pi$ is known. We construct a surrogate action set $\mathcal{X}_\pi = \{ g_\pi(\theta) : \theta \in \Theta \}$, where $g_\pi(\theta) := \mathbb{E}_{\mathcal{A} \sim \pi}\!\left[\arg\max_{a \in \mathcal{A}} \langle a, \theta \rangle \right]$ is the expected greedy action under the stationary distribution $\pi$. This set is supplied to a linear bandit algorithm $\Lambda$, which treats $\mathcal{X}_\pi$ as its fixed arm set.

The interaction proceeds as follows: At round $t$, $\Lambda$ selects a proxy action $g_\pi(\theta_t) \in \mathcal{X}_\pi$. This proxy is \emph{not} played. Instead, the learner observes the current Markovian context $\mathcal{A}_t$ and plays the greedy action $a_t := \arg\max_{a \in \mathcal{A}_t} \langle a, \theta_t \rangle$, receiving reward $r_t$. We enforce a \emph{delayed feedback}: the pair $(g_\pi(\theta_{t-\tau}), r_{t-\tau})$ is revealed to $\Lambda$ only after a delay of $\tau$ rounds. This is summarized in Algorithm~\ref{alg:known}

The delay $\tau$ allows the Markov chain to mix. While $\mathbb{E}[r_t|\theta_t]$ corresponds to the value of the specific context $\mathcal{A}_t$, the linear bandit oracle $\Lambda$ expects rewards corresponding to the stationary average $g_\pi(\theta_t)$. By delaying the update, we ensure that the dependence between the action selection (based on history up to $t-\tau$) and the specific context $\mathcal{A}_t$ decays. This effectively reduces the discrepancy between the observed reward and the stationary expectation to a vanishing bias, which standard linear bandit algorithms can tolerate. We formalize this in the following proposition.

\begin{algorithm}[t]
  \caption{Reduction to Single Context under Known Stationary Distribution}
  \label{alg:known}
  \begin{algorithmic}
    \STATE {\bfseries Input:} Confidence Parameter $\delta$, Single-context linear bandit algorithm $\Lambda$, Feedback delay $\tau$
    
    \FOR{$t = 1:\tau$}
        \STATE Pick $\theta_t$ arbitrarily such that $g_\pi(\theta_t) \in \mathcal{X}_\pi$
        \STATE Play $a_t=\arg\max_{a\in \mathcal{A}_t} \inner{a}{\theta_t}$, obtaining $r_t$.
    \ENDFOR
    
    \FOR{$t = (\tau + 1):T$}
        \STATE Let the action that $\Lambda$ selects be $g_\pi(\theta_t)\in\mathcal{X}_\pi$ after observing action-reward pairs $(g_\pi(\theta_{1}),r_{1}),\dots, (g_\pi(\theta_{t-\tau-1}), r_{t-\tau-1})$
        \STATE Play $a_t=\arg\max_{a\in \mathcal{A}_t}\inner{a}{\theta_t}$ and receive reward $r_t$. Provide $(g_\pi(\theta_{t-\tau}), r_{t-\tau})$ to $\Lambda$.
    \ENDFOR
  \end{algorithmic}
\end{algorithm}


\begin{proposition}\label{prop:delayed-reward}
 When $\tau = \lceil{c_\tau}\log T/(1-\beta) \rceil$, where $c_\tau>1$ is a fixed constant, the reward $r_t$ can be written as
\[
r_t = \inner{g_\pi(\theta_t)}{\theta_\star} + \Delta_t + \eta_t',
\]

where ${|\Delta_t|} \leq 2C_\text{mix}T^{-c_\tau}$ almost surely,  $\E{\eta'_t \mid \mathcal{F}'_t} = 0$ and $\E{\exp(\lambda \eta_t') \mid \mathcal{F}'_t} \leq \exp((17/2)\lambda^2)$ for all $\lambda \in \mathbb{R}$, with $\mathcal{F}'_t = \sigma\{\theta_{1},{r_{1}},\ldots,\theta_{t-\tau}\}$ denoting the filtration of the delayed history up to time $t$.

\end{proposition}

\begin{proof}[Proof Sketch]

Let $\rho$ be a probability distribution over contexts and let $\theta \in \Theta$, and define
\[
g_\rho(\theta) := \mathbb{E}_{\mathcal{A}\sim \rho}\!\Big[ \arg\max_{a\in \mathcal{A}} \langle a,\theta\rangle \mid \theta\Big].
\]
We first show that, for any fixed $\theta$, the Euclidean distance between $g_\rho(\theta)$ and $g_{\rho'}(\theta)$ induced by any two probability measures $\rho$ and $\rho'$ over the contexts is upper bounded by the total variation distance between these two distributions. We then bound the additive bias between the realized greedy action and the stationary surrogate action by the Euclidean distance between the induced $g_{\rho_t}(\theta_t)$ and $g_\pi(\theta_t)$, where $\rho_t$ is the probability measure over contexts conditioned on $\mathcal F'_t$. With the $\tau$ delay, $\mathcal F'_t$ only contains information up to time $t-\tau$, so by the Markov property this conditional distribution can be written as a $\tau$-step pushforward under the transition kernel. Uniform geometric ergodicity then gives $\mathrm{TV}(\rho_t,\pi)\le C_{\mathrm{mix}}\beta^\tau$. After showing that this conditional bias decays exponentially in $\tau$, we center the remaining error term and show that it is still conditionally sub-Gaussian.

\end{proof}

Since $c_\tau>1$, the bias term $\Delta_t$ satisfies $\sup_{t\le T}|\Delta_t| = o(1)$ as $T\to\infty$, and therefore does not affect the regret order of standard linear bandit algorithms. In particular, for UCB-style methods such as \citet{abbasi2011improved}, the contribution of $\Delta_t$ results only in a vanishing additive term in the cumulative regret. We give a more detailed analysis of this behavior in Appendix~\ref{app:linbandits}. For elimination-style algorithms, such as the phased elimination algorithm of \citet{lattimore2020} used in Section~\ref{sec:unknown}, the effect of $\Delta_t$ can be absorbed by enlarging the elimination radius to account for a misspecification level $\varepsilon_T=2C_{\mathrm{mix}}T^{-c_\tau}$, as in Remark~D.1 of \citet{lattimore2020}, or by any larger vanishing radius when $C_{\mathrm{mix}}$ is unknown, so that the final regret stays in the order of $O(d\sqrt{T \log T})$. 

It remains to bound the difference between the regret incurred by the linear bandit algorithm $\Lambda$ and that incurred by Algorithm~\ref{alg:known}.




\begin{theorem}
\label{thm:known-diff}
Consider a context-Markovian linear bandit instance $\mathcal M$ as defined in Section~\ref{sec:not}, and let $L$ denote the associated linear bandit instance with action set $\mathcal{X}_\pi=\{g_\pi(\theta):\theta\in\Theta\}$. For any linear bandit algorithm $\Lambda$, the regret of Algorithm~\ref{alg:known} satisfies, with probability at least $1-\delta$,
\begin{align*}
&\big| R_T^{\Gamma}(\mathcal M) - R_T^\Lambda(L) \big|\le c \left(\sqrt{T \tau\log\frac{\tau}{\delta}}+\tau\right)
\end{align*}
where $\tau$ is set to $\tau =\lceil c_\tau\log T/(1-\beta)\rceil$.  $R_T^\Lambda(L)$ is the regret of $\Lambda$ on $L$, $R_T^{\Gamma}(\mathcal M)$ is the regret of Algorithm~\ref{alg:known} on $\mathcal M$, $\beta \in (0,1)$ is the geometric mixing rate of the underlying Markov chain, and $c>0$ is a universal constant.
\end{theorem}

\begin{proof}[Proof Sketch]
We express the regret difference as a sum of per-round deviations between the stationary feature map \(g_\pi(\cdot)\) and the greedy action induced by the realized context, together with an initial \(\tau\)-round truncation. The analysis proceeds via a \(\tau\)-spaced blocking argument that yields martingale difference sequences, to which the Azuma--Hoeffding inequality is applied. The complete proof is deferred to Appendix~\ref{app:known-diff}.
\end{proof}

This result implies that Algorithm~\ref{alg:known} incurs an additional regret difference of order $O\!\left(\sqrt{T\tau\log(\tau T)}+\tau\right)$ with probability at least $1-1/T$. Since the tightest known high-probability regret bound for linear bandit algorithms is $O(d\sqrt{T\log T})$, it follows that whenever the mixing-dependent term is dominated by the linear bandit regret (e.g., in sufficiently fast-mixing regimes), Algorithm~\ref{alg:known} achieves the same regret order $O(d\sqrt{T\log T})$ with high probability.

In expectation, Algorithm~\ref{alg:known} differs from the underlying linear bandit oracle only by a lower-order additive mixing term. In particular, the mixing dependence enters through the delay parameter~\(\tau\), rather than as a multiplicative inflation on the leading \(\sqrt{T}\) regret term; we prove this in Appendix~\ref{app:expected_coupling}. This distinction is specific to expectation. For high-probability regret control, additive functionals of geometrically ergodic Markov chains necessarily fluctuate on the scale \(O\!\left(\sqrt{T\log(1/\delta)/(1-\beta)}\right)\), since their variance is of order \(\Theta(T/(1-\beta))\). Thus, Bernstein--Freedman--type concentration bounds for Markov chains unavoidably introduce a \((1-\beta)^{-1/2}\) factor on the \(\sqrt{T}\) scale~\citep{paulin_bound}. We state the exact concentration inequality explicitly in Appendix~\ref{app:alg2} for later proofs.

\section{Unknown Stationary Distribution} 

\label{sec:unknown}
We now consider the setting in which the stationary distribution~$\pi$ of the Markov context process is unknown. In this case, the surrogate map $g_\pi(\cdot)$ is unavailable and must be estimated online. We propose Algorithm~\ref{alg:unknown}, which proceeds in epochs, maintaining at each phase~$m$ an empirical surrogate map $g^{(m)}(\cdot)$ constructed from past observations and defined only on a finite $1/T$-net~$\Theta' \subseteq \Theta$. The induced surrogate action set $\mathcal X_m=\{g^{(m)}(\theta):\theta\in\Theta'\}$ is supplied to a misspecification-robust linear bandit algorithm~$\Lambda_{\epsilon_m}$, where $\epsilon_m$ accounts for the approximation error between $g^{(m)}$ and the stationary surrogate $g_\pi$. As the number of samples used to construct $g^{(m)}$ increases across epochs, this misspecification level decays at a controlled rate, allowing $\Lambda_{\epsilon_m}$ to operate with diminishing bias. 
\begin{algorithm}
\caption{Reduction from Markovian Contexts to Single Context (Unknown Stationary Distribution)}
\label{alg:unknown}
\begin{algorithmic}
\STATE {\bfseries Input:} $\delta$, phase lengths $\{t^{(m)}\}_{m=1}^{M+1}$, misspecification-robust linear bandit $\Lambda_\epsilon$, delay $\tau$, $1/T$-net $\Theta'$ over $\Theta$
\STATE \textbf{Initialize:} $g^{(1)}:\Theta'\!\to\!\mathbb R^d$ randomly, $\epsilon_1=1$, $\mathcal X_1=\{g^{(1)}(\theta):\theta\in\Theta'\}$
\FOR{$m=1:M$}
\FOR{$t=(t^{(m)}+1):(t^{(m)}+\tau)$}
\STATE Pick $\theta_t$ arbitrarily such that $g^{(m)}(\theta_t) \in \mathcal{X}_m$
\STATE Play $a_t=\arg\max_{a\in \mathcal{A}_t} \inner{a}{\theta_t}$, obtaining $r_t$.
\ENDFOR
\FOR{$t=(t^{(m)}+\tau+1):\,t^{(m+1)}$}
\STATE Let $g^{(m)}(\theta_t)\in\mathcal{X}_m$ be the arm selected by $\Lambda_{\epsilon_m}$ after observing rewards $r_{t^{(m)}+1},\dots,r_{t-\tau-1}$
\STATE Play $a_t=\arg\max_{a\in\mathcal A_t}\langle a,\theta_t\rangle$, observe $r_t$
\STATE Feed $(g^{(m)}(\theta_{t-\tau}),r_{t-\tau})$ to $\Lambda_{\epsilon_m}$
\ENDFOR
\STATE $g^{(m+1)}(\cdot)=\frac{1}{t^{(m+1)}}\sum_{t=1}^{t^{(m+1)}}\arg\max_{a\in\mathcal A_t}\langle a,\cdot\rangle$
\STATE $\mathcal X_{m+1}=\{g^{(m+1)}(\theta):\theta\in\Theta'\}$
\STATE
Pick $\epsilon_{m+1}$ as the upper bound stated in Lemma \ref{lem:misspec}
\ENDFOR
\end{algorithmic}
\end{algorithm}


\begin{theorem}
\label{thm:unknown_main}
Consider Algorithm~\ref{alg:unknown} with epoch lengths
$T_m=t^{(m+1)}-t^{(m)}=\tau+2^{m-1}$ and delay
$\tau=\lceil c_\tau \log T/(1-\beta)\rceil$.
Let $\Lambda_{\epsilon_m}$ be the PE algorithm described in~\citet{lattimore2020}.
Then, the regret of Algorithm~\ref{alg:unknown} $R_T^\Gamma(\mathcal M)$ on the
context-Markovian instance $\mathcal M$ specified in Section \ref{sec:not} satisfies with probability at least $1-1/T$,
\begin{align*}
R_T^\Gamma(\mathcal M) \le C \Bigg( &d\sqrt{\frac{T \log T}{1-\beta}} \Bigg),
\end{align*}
when the leading term dominates\footnote{The full bound and regime is provided in Appendix \ref{app:full}}, where $C>0$ is a universal constant.
\end{theorem}

Theorem~\ref{thm:unknown_main} quantifies the price of not knowing the stationary distribution~$\pi$: estimating $g_\pi(\cdot)$ from a single Markovian trajectory requires high-probability uniform control of Markov additive functionals, whose fluctuations scale as the aforementioned $\sqrt{\mathrm{Var}(S_T)}=\Theta(\sqrt{T/(1-\beta)})$ under uniform geometric ergodicity, so the regret scales as the best known linear bandit rate with an additional inflation factor of $(1-\beta)^{-1/2}$ induced by the Markovian dependence.


We establish the claim of Theorem~\ref{thm:unknown_main} via three intermediate lemmas in Appendix~\ref{app:alg2}, which directly yield the stated regret bound.
\begin{enumerate}
\item We characterize, for each epoch, the model misspecification induced by approximating $g_\pi(\cdot)$ with $g^{(m+1)}(\cdot)$. (Lemma~\ref{lem:misspec}).
\item We bound, with high probability, the difference between the cumulative regret of the $\Lambda_{\epsilon_m}$ instances and that of Algorithm~\ref{alg:unknown}. (Lemma~\ref{lem:regret-diff})
\item We bound, with high probability, the cumulative regret incurred by all $\Lambda_{\epsilon_m}$ instances invoked within Algorithm~\ref{alg:unknown} (Lemma~\ref{lem:lambda_regret}).
\end{enumerate}

Although explicitly enumerating a \(1/T\)-net \(\Theta'\) can be computationally infeasible when the parameter set \(\Theta\subseteq\mathbb R^d\) is not sufficiently structured, \citet{hanna2023contexts} show that the map \(g^{(m)}(\cdot)\) can be implemented efficiently under standard optimization and linear regression oracle assumptions commonly used in bandit algorithms. We restate this computational result in Appendix~\ref{app:comp}.

\section{Implications} 



We show that both when the stationary distribution of the context process is known and
when it is unknown, a linear bandit oracle can be employed to obtain high probability regret bounds of
order $O\!\left(d\sqrt{T \log T}\right)$ and $O\!\left(d\sqrt{(T \log T)/(1-\beta)}\right)$, respectively, when the leading term dominates. Although our analysis focuses
on Markovian contexts, the approach extends naturally to all ergodic processes governing the
context distribution, provided that convergence to the stationary distribution is sufficiently
fast.

Knowledge about the context distribution allows the analysis of many types of
bandit problems, although in the base case, known algorithms such as
\citet{abbasi2011improved} are already minimax optimal up to log factors and work
even with adversarial context distributions. Instead, stochastic contextual
bandit models can leverage the probability distribution over contexts to obtain
sharper guarantees and address refined problem settings
\citep{hao2020adaptive,wu2020stochastic,tirinzoni2020asymptotically}.
In particular, \citet{hanna2023contexts} and later \citet{hanna2023efficient} have demonstrated that reducing the i.i.d contexts to a single-context instance improves upon the best known regret bounds in various problems. As an illustration, we similarly give two problem variants that demonstrate the gains in the Context-Markovian setting, although the possible applications are broader.

In this section, we treat the mixing rate \(\beta\) as constant for simplicity and state the bounds for the unknown stationary distribution case. The corresponding $\beta$-explicit bounds can be obtained by adding the relevant additive terms to the stated bound in Theorem~\ref{thm:unknown_main}; for the known stationary distribution case, they can be obtained by adding the corresponding terms and the bound from Theorem~\ref{thm:known-diff} to the underlying linear-oracle bound.

\subsection{Contextual Linear Bandits with Misspecification}
\label{subsec:misspec-imp}
Linear bandits rely on the reward being linear in the selected action and the unknown parameter. The $\varepsilon$-misspecified setting relaxes this assumption by defining the reward as
\[
    r^\varepsilon_t
    =
    \langle a_t,\theta_\star\rangle
    +
    \zeta_t(a_t)
    +
    \eta_t,
\]
where $\zeta_t(a_t)$ is an unknown term that captures the misspecification of the linear bandit model and satisfies $\sup_{t\in[T]}\sup_{a\in\mathcal A_t}|\zeta_t(a)| \le \varepsilon$ for a constant $\varepsilon>0$. For contextual linear bandits with changing action sets and unknown $\varepsilon$, \citet{foster2020} obtain the expectation bound $O(d\sqrt T\log T+\varepsilon\sqrt{dT})$. In the i.i.d.\ contextual setting, \citet{hanna2023contexts} use the reduction to single-context linear bandits to improve this bound from expectation to high probability with  $O(d\sqrt{T\log T}+\varepsilon\sqrt{dT}\log T)$ regret. Our reduction extends the same guarantee to context-Markovian linear bandits. The full corollary is given in Appendix~\ref{app:misspec-imp}.

\subsection{Contextual Linear Bandits with Adversarial Corruption}

In the adversarially corrupted setting of \citet{wei2022}, the learner observes
\[
    r_t^C=\langle a_t,\theta_\star\rangle+\eta_t+c_t(a_t), 
    \qquad 
    \sum_{t=1}^T\sup_{a\in\mathcal A_t}|c_t(a)|\le C,
\]
where \(c_t:\mathcal A_t\to\mathbb R\) is a corruption function chosen over the current action set by an adaptive adversary before the learner's current action is realized, and \(C\) is the total corruption budget. The adversary may know the learner's policy and observe the past history, but does not observe the current action before choosing the corruption function. For contextual linear bandits with changing action sets and unknown \(C\), \citet{wei2022} obtain the high-probability bound \(\widetilde O(d^{4.5}\sqrt T+d^4C)\). In the i.i.d.\ contextual setting, \citet{hanna2023contexts} use the reduction to fixed-action linear bandits to improve this to \(\widetilde O(d\sqrt T+d^{3/2}C)\) with high probability. Our reduction also extends the same guarantee to context-Markovian linear bandits, and the full corollary is given in Appendix~\ref{app:corrupt-imp}.

\section{Application: Vehicle Tracking in Sensor Networks}
\label{sec:app}


\subsection{Problem Setup}
\label{subsec:setup}

With the increasing edge capabilities of internet-of-things (IoT) devices, sequential decision-making over networks of densely connected sensing nodes with multiple modalities, such as camera, audio, and vibration sensors, is becoming an emerging application area for collaborative sensing algorithms \citep{He2022Collaborative,Siam2025AIoT,Li2023Heterogeneous,marlin2023iobt}. A central challenge in this setting is the management of resource constraints through the selective use of sensing assets with different capability levels, especially when activating higher-cost and more precise assets requires substantially greater computation \citep{Buyukkalayci2026TopP,Liu2022SensorSelection,Subramaniam2024TrackMDP}.

In target tracking over such sensor fields, many emergent approaches are based on MDP formulations, particularly partially observable Markov decision processes (POMDPs), which maintain a belief state over possible latent states and their transitions \citep{Subramaniam2024TrackMDP,Jiang2023SensorManagement,Lauri2022POMDPRobotics}. However, existing algorithms in this class present major practical challenges for deployment and for our dataset, since POMDP solvers are often computationally infeasible, especially for real-time inference on compute-limited nodes \citep{Lauri2022POMDPRobotics,Subramaniam2024TrackMDP,Singh2023EdgeAI}. Consequently, approaches such as \citet{Buyukkalayci2026TopP,Liu2022SensorSelection} resort to simpler and less powerful models that remain tractable in practice, at the occasional cost of reduced tracking accuracy. In this work, we provide a formulation of this problem in a recommendation-system-style framework, rather than as a reinforcement-learning problem in which actions influence state transitions and thereby expand the search space considerably.

\textbf{Dataset.}
We collected acoustic measurements from an outdoor vehicle-tracking experiment in which an ATV-type vehicle repeatedly traversed a set of predefined trajectories, as illustrated in Figure~\ref{fig:gq_map}. The vehicle moved at an average speed of approximately $3\,\mathrm{m/s}$ for roughly $20$ minutes over an area of about $10{,}000\,\mathrm{m}^2$, which included several buildings and background outdoor noise, primarily wind and ambient environmental sounds. A real-time GPS transmitter was mounted on the vehicle, while NVIDIA Jetson Orin NX-based sensing nodes were deployed across the field with their precise GPS coordinates recorded. Each node was equipped with a compact single-channel USB microphone with a wide pickup pattern, integrated into a hemispherical plastic enclosure. The sensing nodes operated within an IoT framework \citep{marlin2023iobt} and transmitted their measurements over a network using the MQTT protocol. The raw audio streams were processed into signal-energy measurements by averaging over $200$ ms windows, matching the sampling interval of the GPS trajectory data. \footnote{The dataset will be made public upon publication.}

\begin{wrapfigure}{r}{0.5\linewidth}
    \centering
    \includegraphics[width=\linewidth]{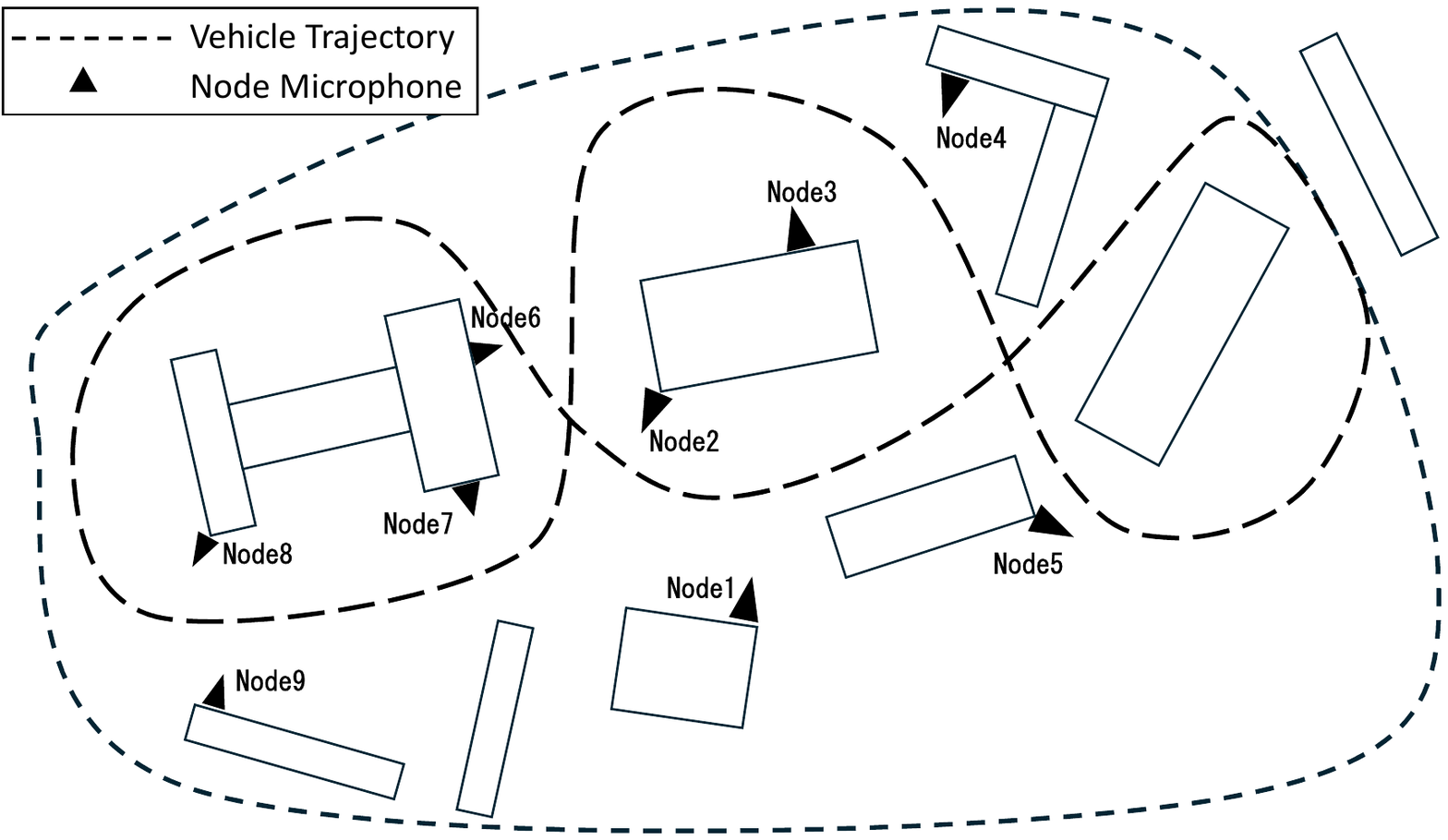}
    \caption{Sketch of the outdoor experiment site and vehicle trajectory.}
    \label{fig:gq_map}
\end{wrapfigure}

\textbf{Objective.}
The goal is to recommend a subset of the available nodes such that higher-cost sensing capabilities can be activated only on the most informative sensors. Hence, we define the objective as a smooth proximity-based utility that rewards subsets containing sensors close to the vehicle. For a subset \(S\) at time \(t\), let \(d_i^t\) denote the distance between node \(i\) and the vehicle, and let
\(d_{(1)}^t(S) \leq \cdots \leq d_{(|S|)}^t(S)\)
denote the sorted distances of the selected nodes. We define
\begin{equation}
u_t(S)
=
\sum_{j=1}^{|S|}
\frac{w_j}{1+d_{(j)}^t(S)/\rho},
\label{eq:utility}
\end{equation}
where \(\rho>0\) is a fixed constant and the weights satisfy \(w_1 \geq w_2 \geq \cdots \geq 0\). In our experiments we use decreasing weights so that the closest selected node contributes most to the utility, while additional selected nodes provide smaller gains.

\textbf{Bandit instance construction.} In order to extract linear embeddings to build a contextual bandit instance, we train a small two-tower model \citep{huang2013} by building node-level acoustic and spatial features from signal strength, short-window temporal summaries, relative energy shares, and node coordinates. After constructing the bandit instance we obtain a mixing rate of $\beta \approx 0.8$ by grouping similar states and cyclically replay the data until the time horizon $T$. It is important to note that ergodicity and mixing are established before the cyclic construction of the data, as they are inherent properties of our dataset stemming from the repeated traversal around the field. We defer the exact construction and bandit environment details to Appendix~\ref{app:realdata-details}.


\subsection{Evaluation}
We demonstrate the performance of our algorithms using the bandit instance constructed in Section~\ref{subsec:setup}. Using the learned linear embeddings, we present results under two choices of ground-truth rewards.

\begin{itemize}
\item \textbf{Linearized rewards.} We generate rewards as
\[
r_t(S)
=
x_t(S)^\top\theta_\star+\eta_t,
\]
where \(x_t(S)\) is the linear embedding obtained from the two-tower model for selecting subset \(S\), and \(\theta_\star\) is the global parameter, also obtained from the two-tower construction and hidden from the algorithms. We add a simulated Gaussian noise \(\eta_t\sim \mathcal N(0,\sigma^2)\) independently at each timestep.
\item \textbf{Utility-function rewards.} We obtain rewards directly from the utility function defined in~\eqref{eq:utility}, so that $r_t(S)=u_t(S)$. In this case, relative to the learned linear embedding, the instance can be written as the misspecified setting in Section~\ref{subsec:misspec-imp},
\[
r_t(S)
=
x_t(S)^\top\theta_\star+\eta_t+\zeta_t(S),
\]
where \(\zeta_t(S)\) is an unknown model-misspecification term induced by approximating the utility function with the two-tower linearized representation.
\end{itemize}

Figure~\ref{fig:alg2_tau_comparison} shows the cumulative regret over multiple runs of the linearized-reward instance, together with the temporally smoothed mean rank of the selected arm with respect to the reward function for Algorithm~\ref{alg:unknown}. We also compare against the no-delay version, which reduces to the i.i.d.-context reduction of~\citet{hanna2023contexts}, and a LinUCB baseline. Algorithm~\ref{alg:unknown} incurs larger regret initially because of the \(\tau\)-round warm-up period and delayed updates. However, after this initial phase, it leverages information across contexts more effectively, identifying high-utility actions more rapidly, and eventually achieves lower cumulative regret than the LinUCB baseline. In contrast, without delayed feedback, bad runs have a stronger effect on the average performance, consistent with the role of delay in our theory. Similar results for Algorithm~\ref{alg:known} when the stationary distribution is known are presented in Appendix~\ref{app:extra_sim}. 

\begin{figure}
\centering
\begin{subfigure}{0.48\linewidth}
    \centering
    \includegraphics[width=\linewidth]{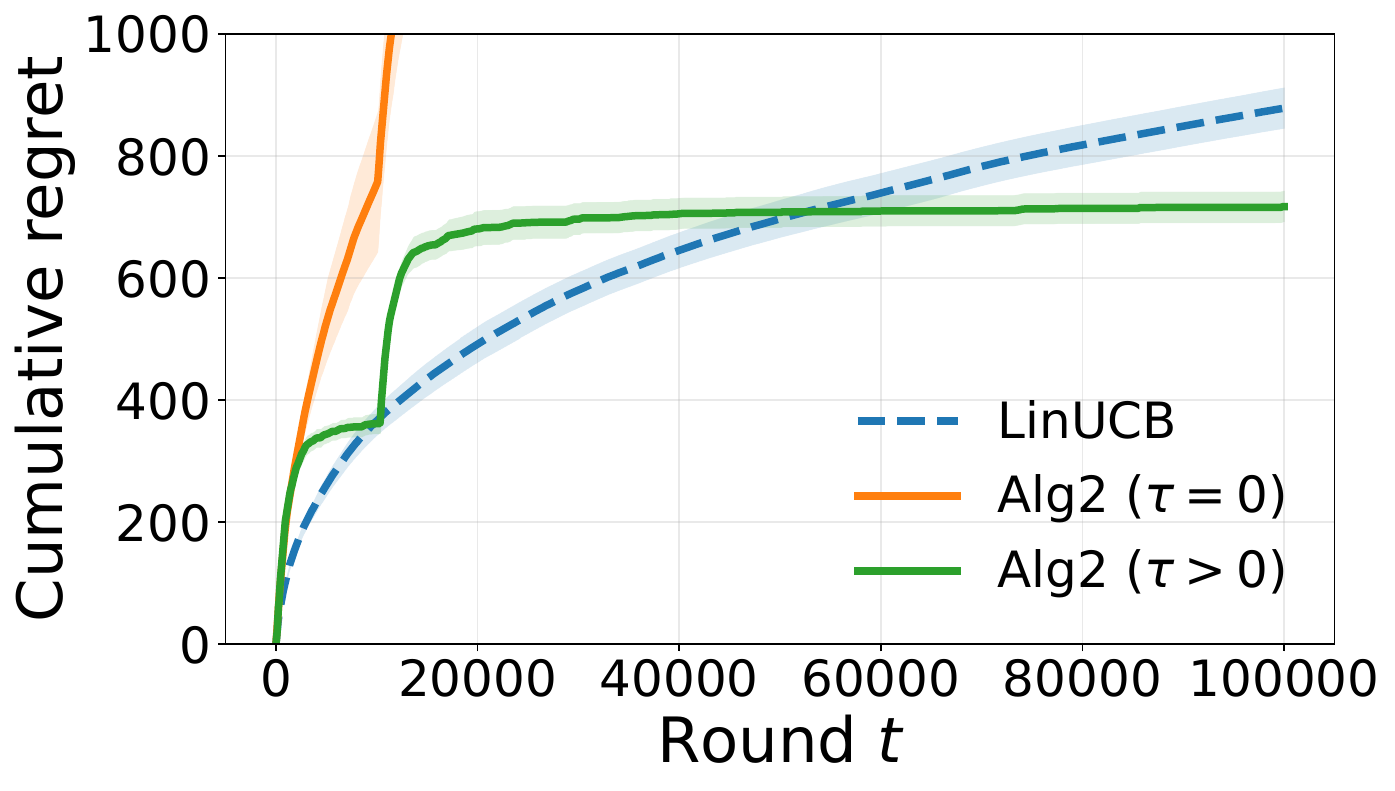}
    \caption{Cumulative regret}
\end{subfigure}
\hfill
\begin{subfigure}{0.48\linewidth}
    \centering
    \includegraphics[width=\linewidth]{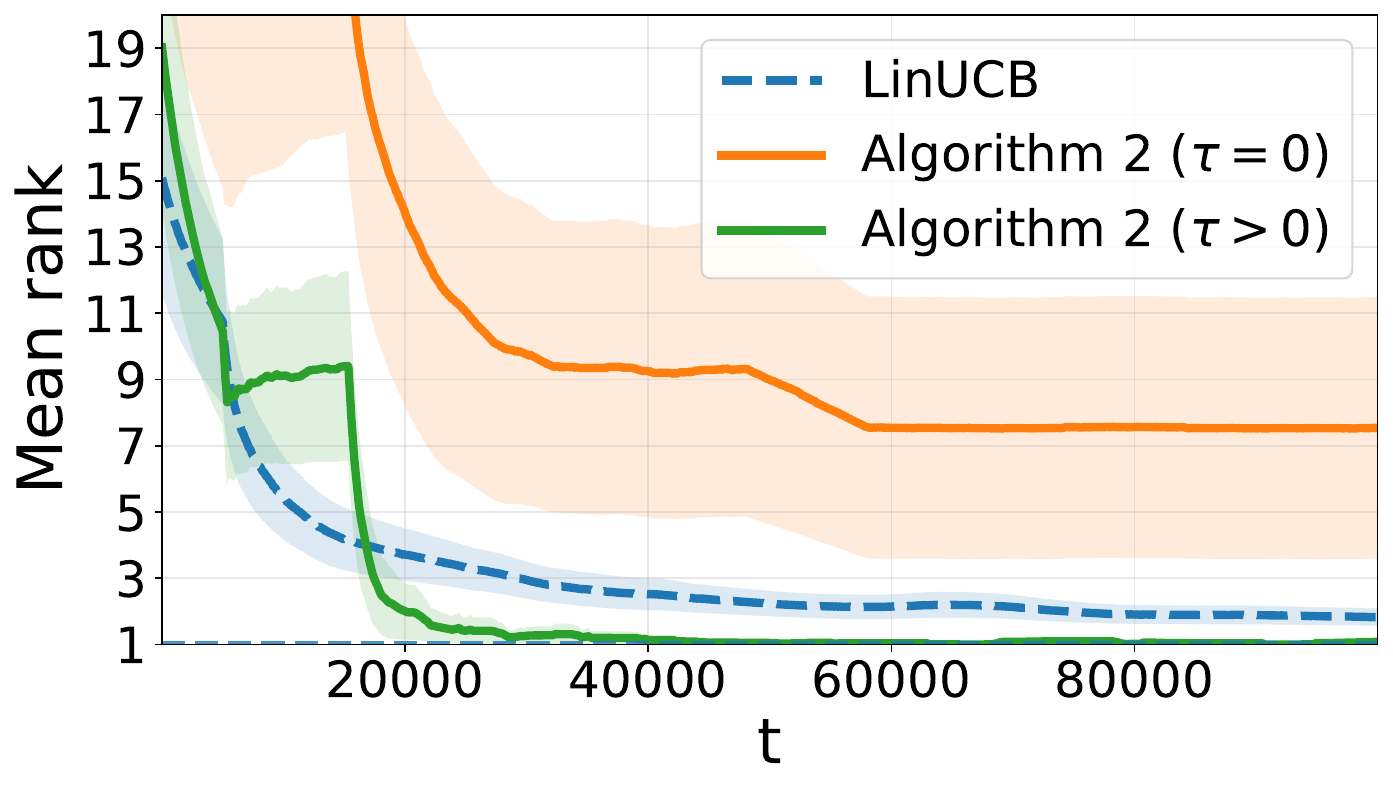}
    \caption{Mean rank of selected action, $|\mathcal A_t|=129$}
\end{subfigure}

\caption{Performance of Algorithm~\ref{alg:unknown} on the bandit instance with linearized rewards. Shaded regions show \(\pm\) one standard error around the mean across runs.}
\label{fig:alg2_tau_comparison}
\end{figure}

Finally, Figure~\ref{fig:alg1-alg2-real-comparison} illustrates Algorithms~\ref{alg:known} and~\ref{alg:unknown} under rewards generated by the utility function $u_t(S)$. As expected, Algorithm~\ref{alg:known} incurs lower regret than Algorithm~\ref{alg:unknown} since it has access to the stationary distribution, while both algorithms eventually outperform the baseline, incurring more regret early on. We give more detail about the practical implementation of the presented algorithms in this section in Appendix~\ref{app:prac_alg}.



\begin{figure}[t]
\centering
\begin{subfigure}{0.48\linewidth}
    \centering
    \includegraphics[width=\linewidth]{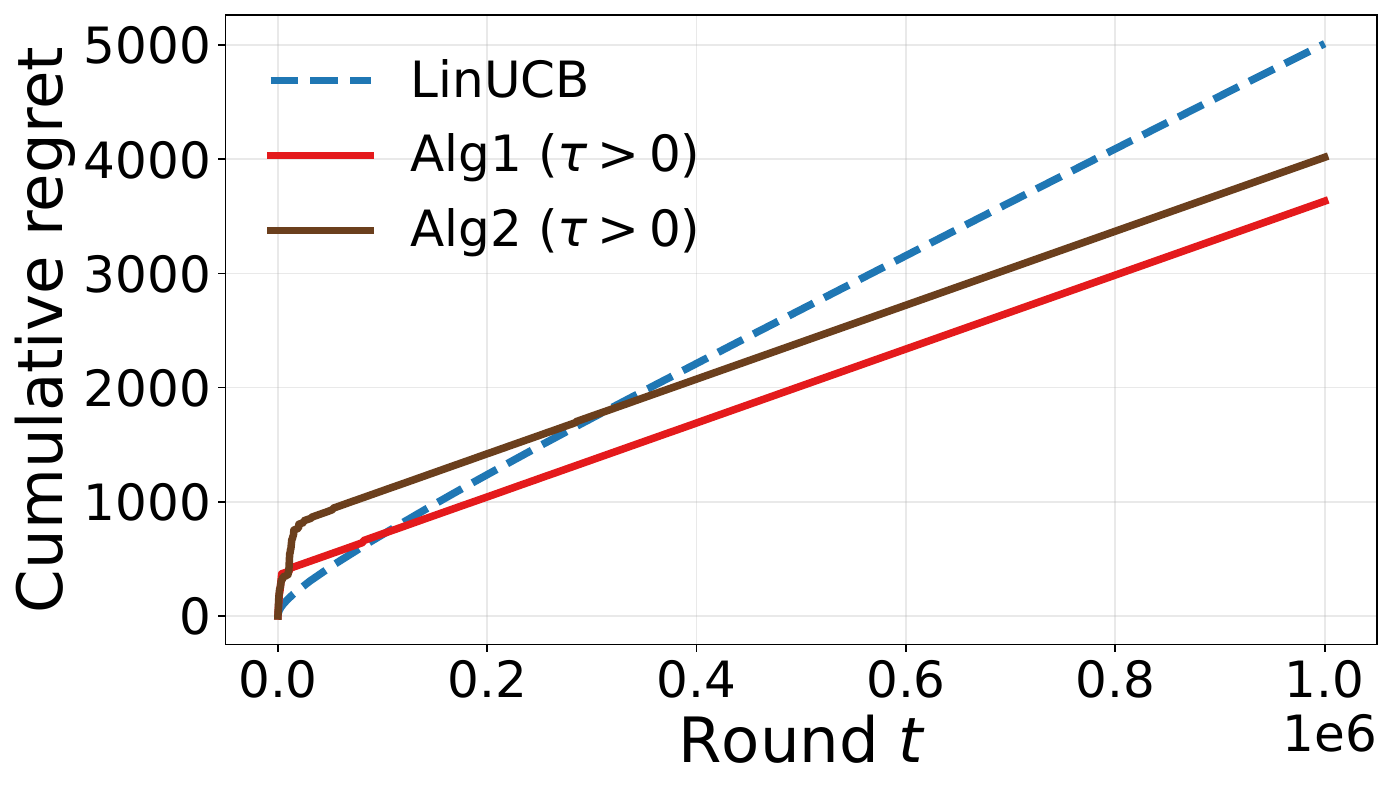}
    \caption{Cumulative regret}
    \label{fig:alg1-alg2-regret}
\end{subfigure}
\hfill
\begin{subfigure}{0.48\linewidth}
    \centering
    \includegraphics[width=\linewidth]{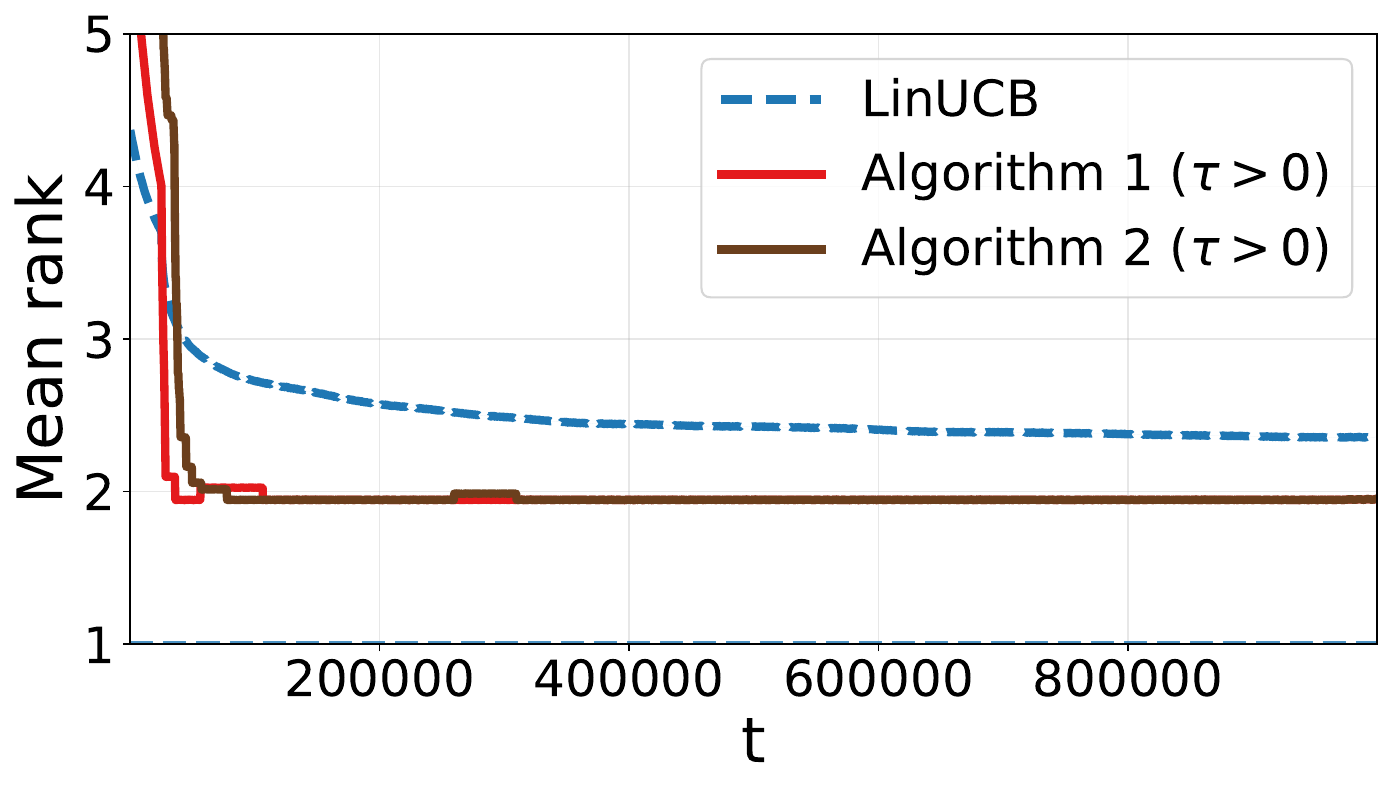}
    \caption{Mean rank of selected action, $|\mathcal A_t|=129$}
    \label{fig:alg1-alg2-rank}
\end{subfigure}

\caption{Performance of Algorithms~\ref{alg:known} and~\ref{alg:unknown} on the bandit instance with rewards equal to the utility function.}

\label{fig:alg1-alg2-real-comparison}
\end{figure}





\newpage
\bibliographystyle{plainnat}
\bibliography{references}


\newpage
\appendix
\onecolumn
\section{Extended Related Work}

\label{app:related}
\paragraph{Linear and contextual linear bandits.}
The ordinary (single-context) linear bandit problem is known to admit worst-case regret of order
$O(d\sqrt{T\log T})$ (e.g., \citealp{lattimore2020}). Such rates are also achieved with high probability by optimism-based algorithms, such as OFUL, and are minimax-optimal up to logarithmic factors according to existing lower bounds
(e.g., \citealp{dani2008stochastic,abbasi2011improved,lattimore2020bandit}). In linear contextual bandits, each round reveals a
context-dependent action set, and rewards are linear in an unknown parameter; canonical algorithms such as LinUCB/OFUL and linear
Thompson sampling achieve $\tilde O(d\sqrt{T})$ regret under sub-Gaussian noise (e.g., \citealp{chu2011contextual,abbasi2011improved,agrawal2013thompson,lattimore2020bandit}). 
While these worst-case rates match those of linear bandits in order, reductions from stochastic contextual models to linear bandits
are highly advantageous because they inherit sharper finite-time guarantees from linear-bandit solvers (often improving constants and
logarithmic factors), simplify the algorithms and analysis, and because they enable the use of linear-bandit techniques in settings where the corresponding contextual
analyses are substantially more involved, such as model misspecification, batched interaction, and robustness to adversarial
corruptions \citep{hanna2023contexts}.

\paragraph{Asymptotically optimal exploration in contextual linear bandits.}
Beyond worst-case optimality, a substantial line of work aims to achieve asymptotically optimal (problem-dependent)
regret by designing exploration policies guided by regret lower bounds, often via primal--dual or saddle-point
reformulations. In the contextual linear bandit setting, \citet{hao2020adaptive} develop exploration that adapts to
benign context distributions and improves asymptotic behavior compared to standard optimistic methods. Building on
lower-bound--driven optimization viewpoints and online learning ideas, \citet{tirinzoni2020asymptotically} propose SOLID, an
asymptotically optimal algorithm with improved computational and finite-time properties under i.i.d.\ contexts.
These works are complementary to ours: they refine exploration when contexts are independent, whereas our focus is
on temporal dependence in the context sets.

\paragraph{Stochastic contextual linear bandits with random action sets and reduction.}
The closest starting point for our results is the recent reduction framework of
\citet{hanna2023contexts}, which study stochastic contextual linear bandits where each
context is a \emph{random set of actions}. They show that, when the context distribution is known, one can reduce the
problem exactly to an ordinary linear bandit by constructing a surrogate action set from the expected argmax action.
When the context distribution is unknown, they reduce to a sequence of \emph{misspecified} linear bandit instances and
retain $\tilde{O}(d\sqrt{T})$-type worst-case regret. The follow up work in \citet{hanna2023efficient}, proposed a computationally efficient implementation of the reduction algorithms. Our work extends this reduction to a strictly more
dependent regime in which the random action sets evolve according to a Markov chain rather than i.i.d.\ sampling.

\paragraph{Adversarial losses, stochastic availability, and reduction-based approaches.}
A different but conceptually related thread studies \emph{stochastic decision sets} (sleeping/stochastically available
actions) under adversarial losses, beginning with online combinatorial optimization with stochastic decision sets
\citep{neu2014online}. In adversarial linear contextual bandits where per-round action sets are drawn from a fixed
distribution, recent work has focused on whether one can bypass access to a context simulator while retaining
near-optimal regret and polynomial-time efficiency \citep{liu2023bypassing,olkhovskaya2023firstsecond}.
Most recently, \citet{vanerven2025improved} extends \citet{hanna2023contexts}'s reduction to reduce adversarial linear contextual bandits
with stochastic action sets to misspecification-robust adversarial linear bandits with fixed action sets, obtaining
$\mathrm{poly}(d)\sqrt{T}$-type regret in polynomial time without knowing the context distribution. While our setting assumes stochastic contexts, our results reinforce the theme that randomness in action availability can be exploited via carefully constructed surrogates to obtain better bounds. Our contribution shows that this principle continues to hold even under Markovian dependence across rounds.

\paragraph{Bandits and reinforcement learning with Markov structure.}
There is extensive literature on bandits with Markovian dynamics, including rested/restless Markovian bandits where
each arm’s state evolves according to a Markov chain (e.g., \citealp{tekin2012online,ortner2012regret,weber1990index,bertsimas2000restless})
and variants with hidden Markov states and side information (e.g., \citealp{yemini2020restless}). These models differ
fundamentally from ours: in our setting rewards remain linear and the dependence enters through the \emph{context/action
set process} itself, which enables a reduction to linear bandits rather than requiring index policies or MDP-style value
learning. On the RL side, structured exploration in MDPs (e.g., \citealp{ok2018exploration}) studies how known structure
in transitions/rewards can reduce exploration costs; our work instead leverages mixing of the observed context process to
control bias in a reduction-based bandit algorithm.


Our setting can also be interpreted as a linear Markov decision process with
\emph{exogenous actions}: the context process $(\mathcal A_t)$ evolves autonomously
according to a Markov kernel $P$, independent of the learner’s actions, while rewards
are linear in the selected action features. In principle, this allows the use of
algorithms for linear MDPs, such as LSVI--UCB~\citep{jin2020provably}. However, generic MDP regret bounds are poorly
suited to this regime, as they incur unnecessary pessimism by accounting for transition
uncertainty that is irrelevant when the dynamics are exogenous. By avoiding a full MDP
treatment and instead reducing the problem to a linear bandit with a carefully controlled
bias via delayed feedback, our analysis recovers regret guarantees that match the best
known rates for linear bandits whenever the context process mixes sufficiently fast.

\paragraph{Latent or partially observed context dynamics and non-stationarity.}
Another related line considers hidden or partially observed state processes that generate contexts and/or rewards.
For example, \citet{nelson2022linearizing} “linearize” contextual bandits with latent state dynamics via online EM/HMM
ideas, and \citet{zeng2024partially} study partially observed, temporally correlated contexts with linear payoffs using
filtering and system-identification techniques. These works address \emph{partial observability}; by contrast, our
contexts are fully observed but temporally dependent. Finally, non-stationary contextual bandits (e.g.,
\citealp{luo2018efficient}) target drifting distributions and optimize dynamic or switching benchmarks; our Markovian
model is stationary but dependent, and our guarantees quantify how mixing controls the reduction error.

\section{Proof of Proposition \ref{prop:delayed-reward}}

\begin{proof} 

Let $\rho$ be a probability distribution over contexts and let $\theta \in \Theta$. Define
\[
g_\rho(\theta) := \mathbb{E}_{\mathcal{A}\sim \rho}\!\Big[ \arg\max_{a\in \mathcal{A}} \langle a,\theta\rangle \mid \theta\Big].
\]


We first show that, for any fixed $\theta$, the Euclidean distance between the corresponding
$g_\rho(\theta)$ induced by any two distributions over the contexts is upper bounded by the total
variation distance between these two distributions.

\begin{lemma}
\label{lem:TV}
For any two probability measures $\rho,\rho'$ on the measurable space of contexts $(\mathsf{S},\mathcal{F}_\mathsf{S})$ and any $\theta \in \Theta$. It holds that
\[
\| g_\rho(\theta) - g_{\rho'}(\theta)\|_2 \;\le\; 2\, \mathrm{TV}(\rho,\rho'),
\]
\end{lemma}


\begin{proof} 
Let $f_{\theta}(\mathcal{A}) := \arg\max\limits_{a\in \mA}\inner{a}{\theta}$.

\[
g_\rho(\theta) - g_{\rho'}(\theta) 
= \int_{\mathsf{S}} f_\theta(\mathcal{A})\,(d\rho-d\rho').
\]
Using the fact that for any $v\in\mathbb{R}^d$, $\|v\|_2 = \sup_{\|w\|_2 \le 1}\langle w,v\rangle$,
\begin{align*}
\| g_\rho(\theta) - g_{\rho'}(\theta)\|_2
&= \sup_{\|w\|_2 \le 1} \left\langle w,\int_{\mathsf{S}} f_\theta(\mathcal{A})\,(d\rho-d\rho')\right\rangle \\
&= \sup_{\|w\|_2 \le 1} \int_{\mathsf{S}} \langle w, f_\theta(\mathcal{A})\rangle\,(d\rho-d\rho') \\
&= \sup_{\|w\|_2 \le 1} \left|\int_{\mathsf{S}} \langle w, f_\theta(\mathcal{A})\rangle\,(d\rho-d\rho')\right|\\
&\leq \sup_{\|w\|_2 \le 1}\int_{\mathsf{S}} \left|\inner{w}{f_\theta(\mathcal{A})}\right| \left|d\rho-d\rho'\right| \\
& \leq \int_{\mathsf{S}}\left|d\rho-d\rho' \right| = 2\mathrm{TV}(\rho,\rho')
\end{align*}
where the last inequality follows from $\|f_\theta(\mathcal{A})\|_2 \le 1$, since $\|a\|_2 \leq 1$ for all $a\in\mathcal{A}$ for all $\mathcal{A} \in \mathsf{S}$ and the Cauchy–Schwarz inequality. 
\end{proof}

Letting $a_t:=\arg\max_{a\in\mathcal{A}_t}\langle a,\theta_t\rangle$, define,
\begin{align*}
b_t
:= 
\inner{a_t}{\theta_\star}
- \langle g_\pi(\theta_t),\theta_\star\rangle,
\qquad
\Delta_t := \mathbb{E}[\,b_t \mid \mathcal{F}'_t\,]
\qquad 
\end{align*}
For $t>\tau$,
\begin{align*}
r_t
= \langle a_t, \theta_\star\rangle + \eta_t
= \langle g_\pi(\theta_t),\theta_\star\rangle
  +
       b_t
  + \eta_t.
\end{align*}

Let $\rho_t(\cdot)
:=
\mathbb{P}(\mathcal{A}_t \in \cdot \mid \mathcal{F}'_t)$
denote the conditional distribution over the contexts at time $t$ generated by
$\mathcal{F}'_t$.  Then
\begin{align*}
\Delta_t
&=
\mathbb{E}\!\left[
    \inner{
        \arg\max_{a\in\mathcal{A}_t}\inner{a}{\theta_t}
        -
        g_\pi(\theta_t)
    }{\theta_\star}
    \;\middle|\;
    \mathcal{F}'_t
\right]
\\[0.25em]
&=
\inner{
    \mathbb{E}\!\left[
        \arg\max_{a\in\mathcal{A}_t}\inner{a}{\theta_t}
        \,\middle|\,
        \mathcal{F}'_t
    \right]
    -
    g_\pi(\theta_t)
}{\theta_\star}\\
&=
\big\langle
    g_{\rho_t}(\theta_t) - g_\pi(\theta_t),
    \theta_\star
\big\rangle,
\end{align*}

Using $\|\theta_\star\|_2\le 1$ and Cauchy-Schwarz inequality,
\[
|\Delta_t|
\le
\| g_{\rho_t}(\theta_t) - g_\pi(\theta_t)\|_2 \le 2\,\mathrm{TV}(\rho_t,\pi).
\]
where the second equality follows from Lemma~\ref{lem:TV}.


For all measurable $A \in \mathcal F_{\mathsf S}$, since the delayed history
$\mathcal F'_t$ contains no information from the context process after time
$t-\tau$, the Markov property implies
\begin{equation}
\label{eq:ctx_markov}
\mathbb{P}(\mathcal{A}_t\in A \mid \mathcal{A}_{t-\tau}, \mathcal{F}'_t)
=
\mathbb{P}(\mathcal{A}_t\in A \mid \mathcal{A}_{t-\tau})
\end{equation}
Let
\[
\rho_t(A)
=
\mathbb{E}\!\left[ P^\tau(\mathcal{A}_{t-\tau},A)\,\middle|\,\mathcal{F}'_t \right].
\]
and
\[
\mu_{t-\tau}(\cdot)
:=
\mathbb{P}(\mathcal{A}_{t-\tau}\in\cdot \mid \mathcal{F}'_t).
\]
Then $\rho_t = P^\tau \mu_{t-\tau}$.

Uniform geometric ergodicity gives, almost surely,
\begin{equation}
\label{eq:TV_ctx}
\mathrm{TV}(\rho_t,\pi)
=
\mathrm{TV}(P^\tau\mu_{t-\tau},\pi)
\le
C_{\mathrm{mix}}\beta^\tau
\end{equation}
Therefore it holds almost surely,
\[
|\Delta_t|
\le 2\mathrm{TV}(\rho_t,\pi)\le
2 C_{\mathrm{mix}}\beta^\tau,
\qquad
\]
Since $-\log\beta \ge 1-\beta$ for $0<\beta<1$, 
\begin{align*}
|\Delta_t|
&\le
2C_{\mathrm{mix}} \beta^\tau \leq 2C_{\mathrm{mix}}\exp{(-\tau(1-\beta))}\leq 2 C_{\mathrm{mix}}\exp{(-c_\tau \log T)} = 2 C_{\mathrm{mix}}T^{-c_\tau}
\end{align*}

Setting
\[
\eta'_t := \eta_t + b_t-\Delta_t.
\]
By construction,
\[
\mathbb{E}[\eta'_t \mid \mathcal{F}'_t]
= \mathbb{E}[\eta_t\mid\mathcal{F}'_t]
  + \mathbb{E}[b_t \mid \mathcal{F}'_t] - \Delta_t
= 0,
\]

Finally, $\eta'_t$ is conditionally sub-Gaussian since $\eta_t$ is conditionally sub-Gaussian given $\mathcal F_t \supseteq \mathcal F'_t$ and $b_t-\Delta_t$ is $\mathcal F'_t$-conditionally mean-zero and bounded. Indeed, $\|a_t\|_2 \le 1$ and $\|g_\pi(\theta_t)\|_2 \le 1$ imply $|b_t| = |\langle a_t - g_\pi(\theta_t), \theta_\star \rangle| \le 2$, and hence $|b_t - \Delta_t| \le 4$. By Hoeffding’s lemma, $b_t-\Delta_t$ is conditionally sub-Gaussian with proxy variance at most $16$, and since $\eta_t$ is conditionally sub-Gaussian with proxy variance $1$, their sum $\eta'_t$ is conditionally sub-Gaussian with proxy variance at most $17$, yielding $\mathbb E[\exp(\lambda \eta'_t)\mid \mathcal F'_t] \le \exp((17/2)\lambda^2)$ for all $\lambda \in \mathbb R$.
\end{proof}
\label{app:TV}

\section{Full Proofs Relating to the Regret Bound of Algorithm \ref{alg:known}}
\subsection{Performance of OFUL in \citet{abbasi2011improved} with Bias of Order $T^{-c_\tau}$}
\label{app:linbandits}

For $t>\tau$, the rewards satisfy
\[
r_t
=
\langle x_t,\theta_\star\rangle
+
\Delta_t
+
\eta_t',
\qquad
\sup_{t\le T}|\Delta_t|\le \varepsilon_T,
\qquad
\varepsilon_T := 2C_{\mathrm{mix}}T^{-c_\tau},
\]
where $c_\tau>1$, $(x_t)$ is $\mathcal F_{t-1}$-measurable, and $\eta_t'$ is conditionally sub-Gaussian given $\mathcal F_{t-1}$.
Following \citet{abbasi2011improved}, define
\[
V_t
=
\lambda I
+
\sum_{s=\tau+1}^t x_sx_s^\top,
\qquad
\hat\theta_t
=
V_t^{-1}
\sum_{s=\tau+1}^t x_s r_s .
\]
Then
\[
\hat\theta_t-\theta_\star
=
V_t^{-1}\sum_{s=\tau+1}^t x_s\eta_s'
+
V_t^{-1}\sum_{s=\tau+1}^t x_s\Delta_s .
\]

By Theorem~1 and Eq.~(5) of \citet{abbasi2011improved}, with probability at least $1-\delta$,
\[
\left\|
V_t^{-1}\sum_{s=\tau+1}^t x_s\eta_s'
\right\|_{V_t}
\le
\beta_t(\delta),
\]
where $\beta_t(\delta)$ is the standard self-normalized confidence radius.
Moreover,
\begin{align*}
\left\|
V_t^{-1}\sum_{s=\tau+1}^t x_s\Delta_s
\right\|_{V_t}
&=
\left\|
\sum_{s=\tau+1}^t x_s\Delta_s
\right\|_{V_t^{-1}} \\
&\le
\varepsilon_T
\sqrt{
\sum_{s=\tau+1}^t x_s^\top V_t^{-1}x_s
}
\le
\varepsilon_T
\sqrt{
\sum_{s=\tau+1}^t x_s^\top V_{s-1}^{-1}x_s
}
\le
\varepsilon_T
\sqrt{
2\log\!\det(V_t/\lambda I)
}.
\end{align*}
where the second inequality uses $V_t^{-1}\preceq V_{s-1}^{-1}$ and the last follows from Lemma~11 of \citet{abbasi2011improved}.
Since $\|x_t\|_2\le1$ and $\lambda\ge1$,
\[
\log\det(V_t/\lambda I)
\le
d\log\!\left(1+\frac{t}{\lambda d}\right)
\le
d\log(1+T).
\]
Hence,
\[
\|\hat\theta_t-\theta_\star\|_{V_t}
\le
\beta_t(\delta)
+
\varepsilon_T\sqrt{2d\log T}.
\]

Running OFUL with this confidence radius and repeating the regret analysis of
Theorem~3 in \citet{abbasi2011improved} yields
\[
R_T^{\mathrm{OFUL}}
\le
c\!\left(
\beta_T(\delta)\sqrt{T}
+
\varepsilon_T\sqrt{dT\log T}
\right)
\]
for a universal constant $c>0$.
Since $\varepsilon_T=O(T^{-c_\tau})$ with $c_\tau>1$, the second term is
$o(\sqrt{T})$, and the regret order remains $O(d\sqrt{T}\log T)$.

\subsection{Proof of Theorem \ref{thm:known-diff}}
\label{app:known-diff}
\begin{proof}
The regret of the linear bandit algorithm is defined as
\begin{equation}
\label{eq:reg_lin}
R_T^{\Lambda}(L) = \sum_{t=\tau}^T \max_{\theta\in\Theta} \langle g_\pi(\theta), \theta_\star\rangle - \langle g_\pi(\theta_t), \theta_\star\rangle,
\end{equation}
while the regret of the contextual bandit algorithm is defined as
\begin{align}
\label{eq:reg_ctx}
R_T^\Gamma(\mathcal M)
&= \sum_{t=1}^T \max_{a\in\mathcal{A}_t}\inner{a}{\theta_\star}
   - \inner{\arg\max_{a \in \mathcal{A}_t}\inner{a}{\theta_t}}{\theta_\star} \\
&\leq 2\tau
   + \sum_{t=\tau}^T \max_{a\in\mathcal{A}_t}\inner{a}{\theta_\star}
   - \inner{\arg\max_{a \in \mathcal{A}_t}\inner{a}{\theta_t}}{\theta_\star}.
\end{align}

where the inequality follows from $\|a\|_2 \leq1, \|\theta\|_2 \leq 1$ for all $a \in \mathcal{A}_t$ and $\theta \in \Theta$ . We first show that, for any context distribution $\rho$,
\begin{equation}
\label{eq:max_of_g}
    \inner{g_\rho(\theta')}{\theta'} = \max_{\theta\in\Theta} \inner{g_\rho(\theta)}{\theta'}, \qquad \forall \theta'\in\Theta. 
\end{equation}
Indeed, for any $\theta',\theta''\in\Theta$,
\begin{align*}
    &\max_{\theta\in\Theta} \inner{g_\rho(\theta)}{\theta'} 
    \ge \inner{g_\rho(\theta')}{\theta'} 
    = \E{\max_{a\in\mathcal{A}_t}\inner{a}{\theta'}} \ge \E{\inner{\arg\max_{a\in\mathcal{A}_t}\inner{a}{\theta''}}{\theta'}} 
    = \inner{g_\rho(\theta'')}{\theta'}. 
\end{align*}
Taking $\theta'' = \arg\max_\theta \inner{g_\rho(\theta)}{\theta'}$ forces equality.

Since $\max_{\theta \in \Theta} \inner{g_\pi(\theta)}{\theta_\star} = \inner{g_\pi(\theta_\star)}{\theta_\star}$, the regret difference can be written as,
\begin{equation}
\begin{aligned}
\label{eq:sum_diff}
&\left| R_T^\Gamma(\mathcal M) - R_T^\Lambda(L) \right| \leq \left| \sum_{t=\tau}^Tb_t(\theta_t)-b_t(\theta_\star)\right| + 2\tau \leq \left| \sum_{t=\tau}^Tb_t(\theta_t) \right| + \left| \sum_{t=\tau}^Tb_t(\theta_\star) \right| + 2\tau
\end{aligned}
\end{equation}
where
\[
b_t(\theta) := \inner{g_\pi(\theta)}{\theta_\star} - \inner{\arg\max_{a \in \mathcal{A}_t}{\inner{a}{\theta}}}{\theta_\star}
\]
Defining the filtration $\mathcal{G}_t := \sigma(\theta_1,\mathcal{A}_1,\dots \theta_{t-\tau},\mathcal{A}_{t-\tau})$, we prove that $\tau$ subsequences of $b_t(\theta_t) - \E{b_t(\theta_t)|\mathcal{G}_t}$ exist whose sums satisfy the martingale property with respect to subsequences of $\mathcal{G}_t$ shifted by $\tau$. For ease of notation, we drop the argument $\theta_t$ and write $b_t$ in the subsequent part to denote $b(\theta_t)$.

Let $t_{r,1}<t_{r,2}< \dots<t_{r,m_r}$ for $r\in \{0,\dots,\tau-1\}$ be the time indices of a subsequence that satisfy $t_{r,k}=t_{r,k-1}+\tau$ for $k\in[m_r]$.

Defining $\Sigma_{t_{r,k}} := \sum_{j=1}^k  \left( b_{t_{r,j}} - \E{b_{t_{r,j}} | \mathcal{G}_{t_{r,j}}}\right)$,
\begin{align*}
&\E{\Sigma_{t_{r,k}} | \mathcal{G}_{t_{r,k-1}+\tau}} 
= \E{\Sigma_{t_{r,k}} | \mathcal{G}_{t_{r,k}}} = \Sigma_{t_{r,k-1}} + \E{b_{t_{r,k}} - \E{b_{t_{r,k}} | \mathcal{G}_{t_{r,k}}} | \mathcal{G}_{t_{r,k}}} = \Sigma_{t_{r,k-1}}
\end{align*}

Since $ \Sigma_{t_{r,k}}$ is $\mathcal{G}_{t_{r,k}+\tau}$-measurable and has bounded increments, by Azuma-Hoeffding inequality, it holds with probability at least $1-\delta/2\tau$ that,
\[|\Sigma_{r,m_r}| \leq c' \sqrt{m_r\log\frac{\tau}{\delta}}\]

for a fixed constant $c'$. By the union bound, it holds for all $r\in\{0,\dots,\tau-1\}$, with probability at least $1-\delta/2$ that,

\begin{align*}
\left|\sum_{r=0}^{\tau-1}\Sigma_{r,m_r}\right|
&\overset{(i)}{\le} \sum_{r=0}^{\tau-1}\left|\Sigma_{r,m_r}\right|
{\le} c' \sqrt{\log\frac{\tau}{\delta}} \sum_{r=0}^{\tau-1}\sqrt{m_r}
\overset{(ii)}{\le} c' \sqrt{\log\frac{\tau}{\delta}} \sqrt{\tau\sum_{r=0}^{\tau-1}m_r} \\
&{=} c' \sqrt{T\tau\log\frac{\tau}{\delta}}
{\le} c'' \sqrt{T \frac{\log T}{1-\beta}\log\frac{\log T/(1-\beta)}{\delta}}.
\end{align*}

for a fixed constant $c''>0$ where ($i$) follows from the triangle inequality and ($ii$) follows from the Cauchy--Schwarz inequality.

Then, with probability at least $1-\delta/2$,

\begin{align*}
\left| \sum_{t=\tau}^T b_t(\theta_t) \right|
&\leq \left|\sum_{t=\tau}^T \big(b_t(\theta_t) - \E{b_t(\theta_t)\mid\mathcal{G}_t}\big)\right|
   + \left|\sum_{t=\tau}^T \E{b_t(\theta_t)\mid\mathcal{G}_t}\right| \\
&\leq \left|\sum_{r=0}^{\tau-1}\Sigma_{r,m_r}\right|
   + \sum_{t=\tau}^T \abs{2C_{\mathrm{mix}}\beta^\tau} \\
&\leq c'' \sqrt{T \frac{\log T}{1-\beta}\log\frac{\log T/(1-\beta)}{\delta}}
   + \frac{2 C_{\mathrm{mix}}}{T^{c_\tau-1}}.
\end{align*}

where the inequality $\left|\sum_{t=\tau}^T\E{b_t(\theta_t)|\mathcal{G}_t}\right| \leq \sum_{t=\tau}^T \abs{2C_{\mathrm{mix}}\beta^\tau}$ follows since $\mathcal{G}_t$ does not contain any information for the context beyond $\mathcal{A}_{t-\tau}$ given $\mathcal{A}_{t-\tau}$, and we can apply the same argument behind equations (\ref{eq:ctx_markov}) and (\ref{eq:TV_ctx}). Bounding $\left| \sum_{t=\tau}^Tb_t(\theta_\star) \right|$ similarly using the filtration $\mathcal{G}'_t := \sigma(\mathcal{A}_1,\dots ,\mathcal{A}_{t-\tau})$ and combining with (\ref{eq:sum_diff}) concludes the proof.

\end{proof}

\subsection{Regret Difference of Algorithm \ref{alg:known} in Expectation}
\label{app:expected_coupling}

\begin{corollary}
\label{thm:expected_coupling}
Consider a context-Markovian linear bandit instance $M$ as defined in Section~\ref{sec:not}, and let $L$ denote the associated linear bandit instance with action set $\mathcal{X}_\pi=\{g_\pi(\theta):\theta\in\Theta\}$. The regret of Algorithm~\ref{alg:known} satisfies,
\[
\Big|\E{R_T^{\Gamma}(M)}-\E{R_T^{\Lambda}(L)}\Big|
\;\le\;
2\tau \;+\;4T\,C_{\mathrm{mix}}\beta^\tau
\]
where $\tau$ is set to $\tau =\lceil c_\tau\log T/(1-\beta)\rceil$ with $c_\tau>1$ thus $\beta^\tau \leq T^{-c_{\tau}}$, and  $R_T^\Lambda(L)$ is the regret of $\Lambda$ on the linear bandit instance $L$, $R_T^{\Gamma}(M)$ is the regret of Algorithm~\ref{alg:known} on the context-Markovian bandit instance $M$.
\end{corollary}

\begin{proof} 
Expressing the regret difference in a manner analogous to Theorem~\ref{thm:known-diff}, we use a similar result to Proposition \ref{prop:delayed-reward} to characterize the bias induced by mixing as an additive term. 
Similar to \eqref{eq:sum_diff}, (\ref{eq:reg_lin}) and (\ref{eq:reg_ctx}) suggest that,
\begin{equation}
\begin{aligned}
\left| \E{R_T^\Gamma(\mathcal M)} - \E{R_T^\Lambda(L)} \right|
&\leq \left| \sum_{t=\tau}^T \big(\E{b_t(\theta_t)} - \E{b_t(\theta_\star)}\big) \right| + 2\tau \\
&\leq \left| \sum_{t=\tau}^T \E{b_t(\theta_t)} \right|
     + \left| \sum_{t=\tau}^T \E{b_t(\theta_\star)} \right|
     + 2\tau.
\end{aligned}
\end{equation}

where,
\[
b_t(\theta) := \inner{g_\pi(\theta)}{\theta_\star} - \inner{\arg\max_{a \in \mathcal{A}_t}{\inner{a}{\theta}}}{\theta_\star}
\]

It follows from the tower property and Jensen's inequality that,
\[
\left|\E{b_t(\theta_t)}\right| = \left|\E{\E{b_t(\theta_t) \mid \mathcal{G}_t} }\right| \leq\E{\left|\E{b_t(\theta_t) \mid \mathcal{G}_t} \right|}
\]
where $\mathcal{G}_t := \sigma(\theta_1,\mathcal{A}_1,\dots \theta_{t-\tau},\mathcal{A}_{t-\tau})$

Since $\mathcal{G}_t$ does not contain any information for the context beyond $\mathcal{A}_{t-\tau}$ given $\mathcal{A}_{t-\tau}$, applying the same argument behind equations (\ref{eq:ctx_markov}) and (\ref{eq:TV_ctx}) shows,

\[
\E{\left|\E{b_t(\theta_t) \mid \mathcal{G}_t} \right|} \leq2\mathrm{TV}(\mu_{t-\tau} P^\tau,\pi)
\le
2C_{\mathrm{mix}}\beta^\tau
\]

where the inner expectation is bounded by a constant almost surely and $\mu_{t-\tau}$ is the probability measure over the contexts, conditioned on $\mathcal{G}_t$. Hence,

\begin{align*}
    \left| \sum_{t=\tau}^T \E{b_t(\theta_t)} \right|\leq \sum_{t=\tau}^T\left| \E{b_t(\theta_t)} \right|\leq \sum_{t=\tau}^T\E{\left|\E{b_t(\theta_t) \mid \mathcal{G}_t} \right|} \leq  2TC_{\mathrm{mix}}\beta^\tau
\end{align*}

Bounding $\left|\E{b_t(\theta_\star)}\right|$ using $\mathcal{G}'_t = \sigma(\mathcal{A}_1,\dots ,\mathcal{A}_{t-\tau})$ with the same argument and summing over $T$ gives the final result.
\end{proof}

\section{Full Proofs Relating to the Regret Bound of Algorithm \ref{alg:unknown}}
\label{app:alg2}

We first state our three lemmas, deferring their proof to subsequent subsections. Before stating our first lemma in this section, we first provide a relevant result from \citet{paulin_bound}.

\textbf{Concentration of Additive Functionals in Uniformly Ergodic Markov Chains.} \citep{paulin_bound} Let $(A_t)_{t\ge 1}$ be a uniformly ergodic Markov chain parametrized by $(C_\mathrm{mix},\beta)$ on a Polish state space $\mathsf{S}$ and let $h:\mathsf{S}\to\mathbb R$ be a function that satisfies $|h(x)|\le 1$ , $\forall x \in \mathsf{S}$. Then for any $T\in\mathbb N$, any $\delta\in(0,1)$, and for any initial distribution of $A_1$, it holds with
probability at least $1-\delta$,

{
\begin{equation}
\label{eq:Pauli-bound}
\left|
\sum_{t=1}^T \left(h(A_t)
-
\mathbb E\!\left[h(A_t)\right]
\right)\right|
\;\le\;
\sqrt{
18\,T\,
\frac{\log\!\left(4C_{\mathrm{mix}}\right)}{1-\beta}
\log\!\left(\frac{2}{\delta}\right)
}.
\end{equation}
}
\noindent This result follows from Corollary~2.10, together with Remark~2.11 in \citet{paulin_bound},
specialized to $\|h\|_\infty\le 1$.

\begin{lemma}
\label{lem:misspec}
For each $m \in [M]$, and $\forall \theta \in \Theta', \forall \theta' \in \Theta$, with probability at least $1-\delta/M$,
{
\begin{align*}
&\left|\inner{g_\pi(\theta)}{\theta'}- \inner{g^{(m)}(\theta)}{\theta'}\right|
\;\le\sqrt{
\frac{
36\,\log\!\left(4C_{\mathrm{mix}}\right)
}{
(1-\beta)\,t^{(m)}
}
\log\!\left(\frac{2M|\Theta'|}{\delta}\right)
}
+\;
\frac{C_{\mathrm{mix}}}{(1-\beta)\,t^{(m)}}
\;+\;
\frac{2}{T}.
\end{align*}
}

\end{lemma}

\begin{proof}[Proof Sketch]
After fixing $\theta\in\Theta'$ and $\theta'\in\Theta$, we decompose the error $\langle g_\pi(\theta)-g^{(m)}(\theta),\theta'\rangle$ into a stochastic fluctuation term, in which we use the borrowed result from \citet{paulin_bound}, and a bias term induced by the nonstationarity of the Markov context process. We extend the bound from $\Theta'$ to all $\theta'\in\Theta$ via a covering argument, introducing only a negligible $O(1/T)$ error. The full proof is covered in Appendix \ref{app:misspec}
\end{proof}

\begin{lemma}
\label{lem:regret-diff}
Let $\Lambda_\epsilon$ be an algorithm for linear bandits with $\epsilon$ misspecification, and let $R_{T_m}^{\Lambda_{\epsilon_m}}(L_m)$ be the regret incurred by $\Lambda_{\epsilon_m}$ in Algorithm \ref{alg:unknown} in epoch $m$ on the misspecified linear bandit instance $L_m$ with $T_m = t^{(m+1)}-t^{(m)}$. With probability at least $1-\delta$ it holds that,
\begin{align*}
 &\abs{R_T^{\Gamma}(M) - R_T^{\Lambda_\epsilon}(L_\epsilon)} \leq c\left(\sqrt{T \frac{\log T}{1-\beta}\log\frac{\log T/(1-\beta)}{\delta}}+\frac{(\log T)^2}{1-\beta}\right)
\end{align*}

where $c>0$ is a universal constant and $R_T^{\Lambda_\epsilon}(L_\epsilon) = \sum_{m=1}^M R_{T_m}^{\Lambda_{\epsilon_m}}(L_m)$ and $R_T^{\Gamma}(\mathcal M)$ is the regret of Algorithm \ref{alg:unknown} on the contextual bandit instance $\mathcal M$ defined in Section \ref{sec:not}.
\end{lemma}

\begin{proof}[Proof Sketch]
The difference is upper bounded via the triangle inequality with the term,
\[
\abs{R_T^{\Gamma}(\mathcal M) - R_T^{\Lambda}(L)} + \abs{R_T^{\Lambda}(L) - R_T^{\Lambda_\epsilon}(L_\epsilon)},
\]
where $R_T^\Lambda(L)$ effectively denotes the regret of a hypothetical algorithm that has access to the stationary distribution. The first term is readily bounded in Section~\ref{sec:known}. The second term corresponds to the discrepancy between accumulated regrets over $\Theta$ and its discretization $\Theta'$, which amounts to a constant when summed over $T$. Full proof is deferred to Appendix \ref{app:regret-diff}.
\end{proof}

\begin{lemma}
\label{lem:lambda_regret}
Let Algorithm~\ref{alg:unknown} be run with epoch lengths
$ T_m := t^{(m+1)} - t^{(m)} =\tau+2^{m-1}$, and delay
$\tau=\lceil c_\tau \log T/(1-\beta)\rceil$.
Let $\Lambda_{\epsilon}$ be the PE algorithm specified in ~\citet{lattimore2020} with $\epsilon_m$ being input as the upper bound in Lemma \ref{lem:misspec}.
Then, with probability at least $1-\delta$, the cumulative regret incurred by all
instances of $\Lambda_{\epsilon_m}$ in Algorithm~\ref{alg:unknown} satisfies
\[
R_T^{\Lambda_\epsilon}
\leq
c\!\left(
d\sqrt{
\frac{T}{1-\beta}
\left(\log T+\frac{1}{d}\log\frac{1}{\delta}\right)
}
+
\frac{\sqrt{d}\,\log T}{1-\beta}
\right).
\]
where $c>0$ is a universal constant.
\end{lemma}

\begin{proof}[Proof Sketch]
The result follows by plugging the misspecification bound from Lemma~\ref{lem:misspec} into the regret guarantee of the PE algorithm given in~\citet{lattimore2020} and summing the resulting bounds over all epochs. The full derivation is provided in Appendix~\ref{app:lambda_regret}.
\end{proof}
\subsection{Proof of Lemma \ref{lem:misspec}}
\label{app:misspec}
\begin{proof}
For fixed $\theta,\theta' \in \Theta'$, let,
\[
h(\mathcal{A}_t) := \inner{\arg\max_{a \in \mathcal{A}_t}{\inner{a}{\theta}}}{\theta'}
\]
Then,
{\small
\begin{align}
\label{eq:misspec}
\left|\inner{g_\pi(\theta)}{\theta'}- \inner{g^{(m)}(\theta)}{\theta'}\right| &= \frac{1}{t^{(m)}}\abs{\sum_{t=1}^{t^{(m)}}\inner{g_\pi(\theta)}{\theta'} - h(\mathcal{A}_t)}\\ \notag&\leq \frac{1}{t^{(m)}}\abs{\sum_{t=1}^{t^{(m)}}{h(\mathcal{A}_t) - \E{h(\mathcal{A}_t)}}} + 
\frac{1}{t^{(m)}}\sum_{t=1}^{t^{(m)}}{\abs{\E{h(\mathcal{A}_t)} - \inner{g_\pi(\theta)}{\theta'}
}}
\end{align}}

The first sum in (\ref{eq:misspec}) can be bounded as described in (\ref{eq:Pauli-bound}) since $h(\mathcal{A}_t) \leq 1 $ for all $a \in \mathcal{A}_t,\theta'\in\Theta$. Hence, with probability at least $1 - \delta/(M|\Theta'|^2)$,
\begin{align}
\label{eq:first-term}
&\abs{\sum_{t=1}^{t^{(m)}}{h(\mathcal{A}_t) - \E{h(\mathcal{A}_t)}}} \leq \sqrt{
36\,t^{(m)}\,
\frac{\log\!\left(4C_{\mathrm{mix}}\right)}{1-\beta}
\log\!\left(\frac{2M|\Theta'|}{\delta}\right)}
\end{align}
where, using the union bound, (\ref{eq:first-term}) holds for $\forall \theta,\theta'\in\Theta'$ with probability at least $1-\delta/M$. The second sum in (\ref{eq:misspec}) can be bounded with Lemma \ref{lem:TV}:
\begin{align}
\label{eq:second-term}
\sum_{t=1}^{t^{(m)}}{\abs{\E{h(\mathcal{A}_t)} - \inner{g_\pi(\theta)}{\theta'}
}} &\leq \sum_{t=1}^{t^{(m)}}\abs{\inner{\E{\arg\max_{a \in \mathcal{A}_t}{\inner{a}{\theta}}}-g_\pi(\theta)} {\theta'}} \notag\\
& \leq\sum_{t=1}^{t^{(m)}}\left\|g_t(\theta)-g_\pi(\theta)\right\|_2 \leq \sum_{t=1}^{t^{(m)}}2C_{\mathrm{mix}}\beta^t \leq \frac{2C_{\mathrm{mix}}}{1-\beta}
\end{align}

Where $g_t(\cdot)$ denotes the $g_\rho(\cdot)$ function evaluated at the marginal law of $\mathcal A_t$ at time $t$. Combining (\ref{eq:misspec}), (\ref{eq:first-term}) and (\ref{eq:second-term}) gives for all $\theta,\theta' \in \Theta'$, with probability at least $1-\delta/M$,
{
\begin{align*}
&\left|\inner{g_\pi(\theta)}{\theta'}- \inner{g^{(m)}(\theta)}{\theta'}\right|
\;\le\sqrt{
\frac{
36\,\log\!\left(4C_{\mathrm{mix}}\right)
}{
(1-\beta)\,t^{(m)}
}
\log\!\left(\frac{2M|\Theta'|}{\delta}\right)
}
+\;
\frac{C_{\mathrm{mix}}}{(1-\beta)\,t^{(m)}}
\end{align*}
}

Finally, picking any $\theta\in\Theta'$ and $\theta'\in \Theta$ guarantees there exists $\phi' \in \Theta'$ such that $\|\theta'-\phi'\|_2 \le 1/T$.
By Cauchy--Schwarz and the triangle inequality,
\begin{align*}
&\left|\inner{g_\pi(\theta)}{\theta'}- \inner{g^{(m)}(\theta)}{\theta'}\right|
\le\left|\inner{g_\pi(\theta)}{\phi'}- \inner{g^{(m)}(\theta)}{\phi'}\right|
+\frac{2}{T}, ~ \forall\theta'\in\Theta,\forall \theta \in \Theta'.
\end{align*}
This concludes the proof.

\end{proof}

\subsection{Proof of Lemma \ref{lem:regret-diff}}
\label{app:regret-diff}
\begin{proof}
The regret $R_T^{\Gamma}(M)$ can be decomposed as,
\begin{align}
\abs{R_T^{\Gamma}(M) - R_T^{\Lambda_\epsilon}(L_\epsilon)} \leq \abs{R_T^{\Gamma}(M) - R_T^{\Lambda}(L)} + \abs{R_T^{\Lambda}(L) - R_T^{\Lambda_\epsilon}(L_\epsilon)}
\end{align}
where $R_T^\Lambda$ is defined as,
\begin{align}
R_T^\Lambda(L)
&=
\sum_{t=1}^T
\max_{\theta \in \Theta}
\inner{g_\pi(\theta)}{\theta_\star}
-
\inner{g_\pi(\theta_t)}{\theta_\star},\notag
\end{align}

By definition, the regret incurred by $\Lambda_\epsilon$ operated in Algorithm \ref{alg:unknown} is,
\begin{align}
R_T^\Lambda(L_\varepsilon)
&=
\sum_{t=1}^T
\max_{\theta \in \Theta'}
\inner{g_\pi(\theta)}{\theta_\star}
-
\inner{g_\pi(\theta_t)}{\theta_\star}.\notag
\end{align}

We first prove that $\abs{R_T^{\Gamma}(M) - R_T^{\Lambda_\epsilon}(L_\epsilon)}$ is bounded by a constant.
We have,
\begin{align}
\big| R_T^\Lambda(L) - R_T^\Lambda(L_\varepsilon) \big|
&=
\sum_{t=1}^T
\Big|
\max_{\theta \in \Theta}
\inner{g_\pi(\theta)}{\theta_\star}
-
\max_{\theta \in \Theta'}
\inner{g_\pi(\theta)}{\theta_\star}
\Big|.
\label{eq:regret-diff-step1}
\end{align}

We now bound the inner difference uniformly in $t$. It is derived in (\ref{eq:max_of_g}) that,
\begin{align*}
\inner{g_\pi(\theta')}{\theta'}
=
\max_{\theta \in \Theta}
\inner{g_\pi(\theta)}{\theta'},
\qquad
\forall \theta' \in \Theta.
\end{align*}
Fix $\theta_\star \in \Theta$. Since $\Theta'$ is a $1/T$-net of $\Theta$, there exists
$\phi \in \Theta'$ such that $\|\theta_\star - \phi\|_2 \le 1/T$.
Using \eqref{eq:max_of_g}, boundedness $\|g_\pi(\theta)\|_2 \le 1$, and the triangle inequality,
\begin{align*}
\max_{\theta \in \Theta}
\inner{g_\pi(\theta)}{\theta_\star}
&=
\inner{g_\pi(\theta_\star)}{\theta_\star}
\\
&\le
\inner{g_\pi(\theta_\star)}{\phi}
+
\frac{1}{T}
\\
&\le
\max_{\theta \in \Theta}
\inner{g_\pi(\theta)}{\phi}
+
\frac{1}{T}
\\
&=
\max_{\theta \in \Theta'}
\inner{g_\pi(\theta)}{\phi}
+
\frac{1}{T}
\\
&\le
\max_{\theta \in \Theta'}
\inner{g_\pi(\theta)}{\theta_\star}
+
\frac{2}{T}.
\end{align*}
By symmetry, the reverse inequality also holds, hence
\begin{align}
\Big|
\max_{\theta \in \Theta}
\inner{g_\pi(\theta)}{\theta_\star}
-
\max_{\theta \in \Theta'}
\inner{g_\pi(\theta)}{\theta_\star}
\Big|
\le
\frac{2}{T}.
\label{eq:pointwise-gap}
\end{align}

Substituting \eqref{eq:pointwise-gap} into \eqref{eq:regret-diff-step1} yields
\begin{align}
\big| R_T^\Lambda(L) - R_T^\Lambda(L_\varepsilon) \big|
\le
\sum_{t=1}^T \frac{2}{T}
= 2,
\end{align}

For the difference $\abs{R_T^{\Gamma}(M) - R_T^{\Lambda}(L)}$, we use the result of Theorem \ref{thm:known-diff} after accounting for the $\tau$-length intervals where actions are selected randomly. We have, by the triangle inequality,

\begin{align}
\label{eq:diff_int_sums}
\abs{R_T^{\Gamma}(M) - R_T^{\Lambda}(L)} \leq \sum_{t \in I_r} \abs{R_t^{\Gamma}(M) - R_t^{\Lambda}(L)} + \abs{\sum_{t \in I_t}\left(R_t^{\Gamma}(M) - R_t^{\Lambda}(L)\right)}
\end{align}

where $I_r = \bigcup_{m=1}^M \{t^{(m)}+1,\dots,t^{(m)}+\tau\}$ denotes the set of time indices where $\theta_t$ was picked randomly and $I_t = \bigcup_{m=1}^M \{t^{(m)}+\tau + 1,\dots,t^{(m+1)}\}$ denotes the remaining time indices. $R_T^\Gamma(\mathcal M)$ and $R_t^\Lambda(L)$ denote the per-round regret of Algorithm \ref{alg:unknown} and $\Lambda$, respectively. Note that when $t\in I_t$, the definition of $R_t^\Lambda$ is identical to the regret of the linear bandit algorithm $\Lambda$ that is aware of the stationary distribution used in Algorithm \ref{alg:known}.

By Theorem \ref{thm:known-diff}, we have, with probability at least $1-\delta/2$,
\begin{align}
\label{eq:int_t}
&\big| \sum_{t\in I_t} \left(R_t^{\Gamma}(M) - R_t^\Lambda(L)\right) \big| \le c' \left(\sqrt{T \frac{\log T}{1-\beta}\log\frac{\log T/(1-\beta)}{\delta}}+\frac{\log T}{1-\beta}\right)
\end{align}

where $c'>0$ is a universal constant. For the second term in \eqref{eq:diff_int_sums}, since there are $M=O(logT)$ epochs by the epoch schedule $t^{(m)}=2^{m-1} +\tau$ and the per-step regret is bounded by $1$, it holds almost surely,
\begin{align}
\label{eq:int_rand}
\sum_{t \in I_r} \abs{R_t^{\Gamma}(M) - R_t^{\Lambda}(L)} \leq 2 M\tau\leq c'' \left(\frac{(logT)^2}{1-\beta}\right)
\end{align}

for a constant $c''>0$. Combining \eqref{eq:int_rand} and \eqref{eq:int_t} with \eqref{eq:diff_int_sums} completes the proof.
\end{proof}

\subsection{Proof of Lemma \ref{lem:lambda_regret}}
\label{app:lambda_regret}

\begin{proof}
Let $T_m=\tau+2^{m-1}$, $H_m=2^{m-1}$, and $t^{(m)}=(m-1)\tau+(2^{m-1}-1)$. Throughout the proof, let each appearance of the symbol $c$ denote a (possibly different) positive constant. For each epoch $m\in[M]$, define the misspecification event
\[
\mathcal E_m \;:=\;\Big\{
\forall \theta\in\Theta',\ \forall \theta'\in\Theta:\ 
\big|\langle g_\pi(\theta),\theta'\rangle-\langle g^{(m)}(\theta),\theta'\rangle\big|
\le \epsilon_m
\Big\},
\]
where $\epsilon_1:=1$ and for $m\ge 2$,
\[
\epsilon_m
:=
\sqrt{
\frac{
36\,\log\!\left(4C_{\mathrm{mix}}\right)
}{
(1-\beta)\,t^{(m)}
}
\log\!\left(\frac{2M|\Theta'|}{\delta}\right)
}
+\;
\frac{C_{\mathrm{mix}}}{(1-\beta)\,t^{(m)}}
\;+\;
\frac{2}{T}.
\]
By Lemma~\ref{lem:misspec}, $\Pr(\mathcal E_m)\ge 1-\delta/M$ for each $m$, hence by a union bound,
\[
\Pr\Big(\bigcap_{m=1}^M \mathcal E_m\Big)\;\ge\;1-\delta.
\]
In epoch $m$, Algorithm~\ref{alg:unknown} runs a fresh instance of
$\Lambda_{\epsilon_m}$ and feeds it exactly $H_m$ action--reward pairs. The guarantee of the PE algorithm specified in ~\citet{lattimore2020} then gives,
with probability at least $1-\delta/(2M)$,
\[
R_{H_m}^{\Lambda_{\epsilon_m}}
\le
c\Big(
\sqrt{dH_m\log\!\Big(\tfrac{2M|\Theta'|}{\delta}\Big)}
+
\sqrt{d}\,H_m\,\epsilon_m
\Big).
\]
Taking a union bound over $m\in[M]$ for these bandit-regret events and intersecting with $\bigcap_{m=1}^M\mathcal E_m$,
we obtain that, with probability at least $1-3\delta/2$,
\begin{equation}
\label{eq:epoch_bd_concise}
R_T^{\Lambda_\epsilon}
=\sum_{m=1}^M R_{H_m}^{\Lambda_{\epsilon_m}}
\le
c\sqrt{d\log\!\Big(\tfrac{2M|\Theta'|}{\delta}\Big)}\sum_{m=1}^M \sqrt{H_m}
+
c\sqrt{d}\sum_{m=1}^M H_m\,\epsilon_m.
\end{equation}
Since $H_m=2^{m-1}$ and $\sum_{m=1}^M H_m=2^M-1\le T$, we have
\[
\sum_{m=1}^M \sqrt{H_m}
=
\sum_{m=1}^M 2^{(m-1)/2}
\le
c\,2^{M/2}
\le
c\sqrt{T}.
\]

Next, $\epsilon_1=1$ so $H_1\epsilon_1\le 1$. For $m\ge2$, $t^{(m)}\ge 2^{m-1}-1\ge \tfrac12 H_m$, hence
$H_m\sqrt{1/t^{(m)}}\le \sqrt{2}\sqrt{H_m}$ and $H_m/t^{(m)}\le 2$, so
\[
\sum_{m=2}^M H_m\sqrt{\frac{1}{t^{(m)}}}
\le
c\sum_{m=2}^M \sqrt{H_m}
\le
c\sqrt{T},
\qquad
\sum_{m=2}^M \frac{H_m}{t^{(m)}}\le cM,
\qquad
\sum_{m=1}^M H_m\frac{1}{T}\le 1.
\]
Therefore, applying the result of Lemma \ref{lem:misspec}
\[
\sum_{m=1}^M H_m\,\epsilon_m
\le
c\sqrt{T}\sqrt{\frac{\log(4C_{\mathrm{mix}})}{1-\beta}\log\!\Big(\tfrac{2M|\Theta'|}{\delta}\Big)}
+
c\,\frac{C_{\mathrm{mix}}}{1-\beta}\,M
+
c,
\]

Moreover, it is well known that if $\Theta\subseteq\{\theta\in\mathbb{R}^d:\|\theta\|_2\le 1\}$, the $1/T$-net
$\Theta'\subseteq\Theta$ satisfies $|\Theta'|\le (6T)^d$, so that
\[
\log\!\Big(\tfrac{2M|\Theta'|}{\delta}\Big)
= O\!\big(d\log T+\log\tfrac{1}{\delta}\big).
\]
With $M=\Theta(\log T)$ and absorbing $C_{\mathrm{mix}}$ into constants, this yields
\[
R_T^{\Lambda_\epsilon}
=
O\!\Bigg(
\sqrt{\frac{dT}{1-\beta}}
\sqrt{d\log T+\log\tfrac{1}{\delta}}
\;+\;
\frac{\sqrt{d}\,\log T}{1-\beta}
\Bigg).
\]
\end{proof}
\subsection{High Probability Regret Bound of Algorithm \ref{alg:unknown}}
\label{app:full}
Summing the terms obtained in Lemmas \ref{lem:regret-diff} and \ref{lem:lambda_regret} gives the final high probability regret bound,

\begin{align*}
R_T^\Gamma(\mathcal M) \le C \Bigg( &d\sqrt{\frac{T \log T}{1-\beta}} + \frac{\sqrt{d}\,\log T}{1-\beta} + \sqrt{\frac{T\log T}{1-\beta} \log\left(\frac{T\log T}{1-\beta}\right)} + \frac{(\log T)^2}{1-\beta} \Bigg),
\end{align*}

where $C>0$ is a universal constant. The first term is the leading term whenever
\[
d \ge
\max\left\{
\frac{\log T}{(1-\beta)T},
\sqrt{\log\!\left(\frac{T\log T}{1-\beta}\right)},
\frac{(\log T)^{3/2}}{\sqrt{(1-\beta)T}}
\right\}.
\]
In particular, up to lower-order $\mathrm{polylog}(T)/T$ terms, it is enough that
\[
d \gtrsim
\sqrt{\log\!\left(\frac{T\log T}{1-\beta}\right)}.
\]
which holds practically unless the horizon is impractically large or mixing is extremely slow. Thus the bound becomes dominated by the $O\left(d\sqrt{T\log T/(1-\beta)}\right)$ term.

\section{Computational Complexity}
\label{app:comp}

The reduction framework introduces two primary computational challenges: optimizing over the surrogate action set $\mathcal{X}_m=\{g^{(m)}(\theta):\theta\in\Theta'\}$ for a given direction $u$, and computing the inverse mapping to recover the underlying parameter $\theta$ from a selected surrogate action $g^{(m)}(\theta)$. An exhaustive search over a sufficiently fine discretization of $\Theta$ would be computationally prohibitive. However, this issue can be efficiently resolved using standard optimization and linear regression oracles, as established in \citet{hanna2023contexts}.

Assume access to a contextual linear optimization oracle \(\mathcal{O}(\mathcal{A}; u)\) that returns an optimal action for any realized context and direction, satisfying \(\mathcal{O}(\mathcal{A}; u)\in\arg\max_{a\in\mathcal{A}}\langle a,u\rangle\). For any epoch \(m\), conditional on the observed context sequence \(\mathcal A_1,\ldots,\mathcal A_{t^{(m)}}\), the empirical surrogate map can be evaluated using this oracle as \(g^{(m)}(\theta)=\frac{1}{t^{(m)}}\sum_{s=1}^{t^{(m)}}\mathcal{O}(\mathcal{A}_s;\theta)\).

To construct an approximate optimization oracle over \(\mathcal{X}_m\), we define a discretization parameter \(q>0\) and a projection operation \([\theta]_q = \frac{q}{\sqrt{d}}\lfloor\theta\sqrt{d}/q\rfloor\) applied element-wise. The discrete net is given by \(\Theta'=\{[\theta]_q:\theta\in\Theta\}\). Furthermore, since \(\Theta\subseteq\{\theta\in\mathbb R^d:\|\theta\|_2\le1\}\), the size of this net is bounded by \(|\Theta'|\le C d(1/q)^{4d+2}\).

\begin{lemma}
\label{lem:approx-reduced-oracle}
Consider a given \(m\in[M]\), and let \(
\mathcal X_{m,q}=\{g^{(m)}(\theta):\theta\in\Theta_q\},~
\Theta_q=\{[\theta]_q:\theta\in\Theta\}.
\)
For any \(\theta\in\mathbb R^d\) with \(\|\theta\|_2\le 1\), and any tolerance \(\epsilon>0\), if \(q\le\epsilon/2\), then
\begin{equation}
\nonumber
\left\langle g^{(m)}([\theta]_q),\theta\right\rangle
\ge
\sup_{x\in\mathcal X_{m,q}}\langle x,\theta\rangle-\epsilon .
\end{equation}
\end{lemma}
\begin{proof}
If \(\theta=0\), the claim is trivial. Otherwise, let \(\bar\theta=\theta/\|\theta\|_2\). By the definition of the discretization, \(\|\bar\theta-[\bar\theta]_q\|_2\le q\). Since the base actions satisfy \(\|a\|_2\le1\), we also have \(\|g^{(m)}(\phi)\|_2\le1\) for every \(\phi\in\Theta_q\).

For any \(\phi\in\Theta_q\), oracle optimality gives
\begin{align*}
\left\langle g^{(m)}([\bar\theta]_q),[\bar\theta]_q\right\rangle
&=
\frac{1}{t^{(m)}}\sum_{s=1}^{t^{(m)}}
\left\langle \mathcal O(\mathcal A_s;[\bar\theta]_q),[\bar\theta]_q\right\rangle \\
&\ge
\frac{1}{t^{(m)}}\sum_{s=1}^{t^{(m)}}
\left\langle \mathcal O(\mathcal A_s;\phi),[\bar\theta]_q\right\rangle
=
\left\langle g^{(m)}(\phi),[\bar\theta]_q\right\rangle .
\end{align*}
Therefore, for any \(\phi\in\Theta_q\),
\begin{align*}
\left\langle g^{(m)}([\bar\theta]_q),\bar\theta\right\rangle
&=
\left\langle g^{(m)}([\bar\theta]_q),[\bar\theta]_q\right\rangle
+
\left\langle g^{(m)}([\bar\theta]_q),\bar\theta-[\bar\theta]_q\right\rangle \\
&\ge
\left\langle g^{(m)}([\bar\theta]_q),[\bar\theta]_q\right\rangle-q \\
&\ge
\left\langle g^{(m)}(\phi),[\bar\theta]_q\right\rangle-q \\
&\ge
\left\langle g^{(m)}(\phi),\bar\theta\right\rangle-2q .
\end{align*}
Taking the supremum over \(\phi\in\Theta_q\), using \(\mathcal X_{m,q}=\{g^{(m)}(\phi):\phi\in\Theta_q\}\), and multiplying by \(\|\theta\|_2\), we obtain
\[
\left\langle g^{(m)}\!\left(\left[\frac{\theta}{\|\theta\|_2}\right]_q\right),\theta\right\rangle
\ge
\sup_{x\in\mathcal X_{m,q}}\langle x,\theta\rangle-2q\|\theta\|_2 .
\]
Since \(q\le\epsilon/2\), this gives
\[
\left\langle g^{(m)}\!\left(\left[\frac{\theta}{\|\theta\|_2}\right]_q\right),\theta\right\rangle
\ge
\sup_{x\in\mathcal X_{m,q}}\langle x,\theta\rangle-\epsilon\|\theta\|_2 .
\]
\end{proof}

Consequently, approximate maximization over the surrogate set \(\mathcal{X}_{m,q}\) in direction \(\theta\) only requires evaluating \(g^{(m)}\) at the single discretized point \([\theta/\|\theta\|_2]_q\). This evaluation requires \(t^{(m)}\) calls to the original contextual oracle \(\mathcal{O}(\mathcal{A}_s;\cdot)\).

The recovery step is handled by storage. If all actions played by the linear bandit algorithm are outputs of the reduced optimization oracle, then whenever the reduced oracle evaluates \(g^{(m)}(\theta)\), we store the pair \((g^{(m)}(\theta),\theta)\). If the linear bandit subroutine later selects this stored surrogate action, the reduction retrieves the associated \(\theta\) and executes \(\arg\max_{a\in\mathcal{A}_t}\langle a,\theta\rangle\) in the realized context.

\section{Implications of the Reduction In Markovian Contexts}

Table~\ref{tab:variant-comparison} provides the regret comparison of our proposed algorithms in the Context-Markovian setting to the otherwise best known contextual bandit algorithm baselines that do not exploit temporal correlations. For this section, the mixing rate $\beta$ is treated as constant.

\begin{table}
\centering
\caption{Comparison between the known best guarantees for representative contextual linear bandit variants and the guarantees obtained by applying our Markovian reduction. Here, ``w.h.p.'' denotes high-probability bounds and ``exp.'' denotes bounds in expectation.}
\label{tab:variant-comparison}
\begin{tabularx}{\linewidth}{@{}l>{\raggedright\arraybackslash}Xc>{\raggedright\arraybackslash}Xc@{}}
\toprule
& \multicolumn{2}{c}{Known best bound} 
& \multicolumn{2}{c}{Our Markovian reduction} \\
\cmidrule(lr){2-3}
\cmidrule(lr){4-5}
Variant 
& Bound 
& Type 
& Bound 
& Type \\
\midrule

Base
& $O(d\sqrt{T}\log T)$ \newline \citep{abbasi2011improved}
& w.h.p.
& $O(d\sqrt{T\log T})$
& w.h.p. \\

\addlinespace

Misspecified
& $O(d\sqrt{T}\log T+\varepsilon\sqrt{dT})$ \newline \citep{foster2020}
& exp.
& $O(d\sqrt{T\log T}+\varepsilon\sqrt{dT}\log T)$
& w.h.p. \\

\addlinespace

Adv.\ corrupt.
& $\tilde O(d^{4.5}\sqrt{T}+d^4C)$ \newline \citep{wei2022}
& w.h.p.
& $\tilde O(d\sqrt{T}+d^{1.5}C)$
& w.h.p. \\

\bottomrule
\end{tabularx}
\end{table}

\subsection{Misspecified Contextual Linear Bandits}
\label{app:misspec-imp}
\begin{corollary}
\label{cor:bounded-misspec}
Consider the context-Markovian linear bandit instance defined in
Section~\ref{sec:not}, but suppose that the learner observes rewards
\[
    r_t^{\varepsilon}=r_t+\zeta_t(a_t),
    \qquad
    \sup_{t\in[T]}\sup_{a\in\mathcal A_t}|\zeta_t(a)|\le \varepsilon ,
\]
where \(\zeta_t\) accounts for unknown model misspecification. Then
Algorithm~\ref{alg:unknown}, with the PE algorithm of \citet{lattimore2020} as the subroutine \(\Lambda_\epsilon\), achieves with probability at least \(1-1/T\),
\[
    R_T^{\Gamma,\varepsilon}(\mathcal M)
    \le
    c \left(d\sqrt{T \log T} 
    +
    \,\varepsilon\sqrt{dT}\log T\right),
\]
where \(c>0\) is a universal constant and $R_T^{\Gamma,\varepsilon}(\mathcal M)$ is the regret of the misspecified instance. If \(\varepsilon\) is known, the additional
\(\log T\) factor in the second term can be removed.
\end{corollary}

\begin{proof}
Let \(\epsilon_m\) denote the epoch-wise surrogate misspecification bound from Lemma~\ref{lem:misspec}. By Lemma~\ref{lem:misspec}, in epoch
$m$, the surrogate error induces misspecification at most $\epsilon_m$. Under the
perturbed rewards, the total misspecification becomes at most $\epsilon_m+\varepsilon$. Therefore,
going through the steps of Lemma~\ref{lem:lambda_regret} with
$\epsilon_m+\varepsilon$ in place of $\epsilon_m$ directly achieves the stated result. As described in \citep{lattimore2020}, if $\varepsilon$ is known $\epsilon_m+\varepsilon$ can be incorporated directly into
the confidence intervals of the phased elimination subroutine to remove the $\log T$ factor in the additive term.
\end{proof}

\subsection{Contextual Linear Bandits with Adversarial Corruption}
\label{app:corrupt-imp}

\begin{corollary}
\label{cor:adversarial-corruption}
Consider the context-Markovian linear bandit instance in Section~\ref{sec:not}, but suppose the learner observes corrupted rewards
\[
    r_t^C=r_t+c_t(a_t),
    \qquad
    \sum_{t=1}^T \sup_{a\in\mathcal A_t}|c_t(a)| \le C .
\]
Here \(c_t:\mathcal A_t\to\mathbb R\) is chosen by an adaptive adversary before the learner's current action is realized, and \(C\) is unknown to the learner. Let \(I_m\) denote the set of original time indices whose rewards are used as delayed feedback by the linear-bandit subroutine in epoch \(m\), and define
\[
    C_m := \sum_{t\in I_m} \sup_{a\in\mathcal A_t}|c_t(a)|.
\]
Then \(\sum_{m=1}^M C_m\le C\).

Run Algorithm~\ref{alg:unknown}, replacing the subroutine \(\Lambda_{\epsilon_m}\) in epoch \(m\) by G-COBE~\citep{wei2022} with PE as its base learner. The PE base learners are instantiated with candidate epoch corruption budgets \(C'_m\). Let \(H_m=|I_m|\), set \(\bar\epsilon_m:=\epsilon_m+2C_{\rm mix}T^{-c_\tau}\), and enlarge the PE elimination radius so that
\[
\gamma_m
\ge
2\left[
2\sqrt d\,\bar\epsilon_m
+\sqrt{\frac{4d}{H_m}\log\!\frac{M|\Theta'|}{\delta}}
+\frac{2C'_m(4d\log\log d+18)}{H_m}\sqrt{8d}
+\epsilon_m
\right],
\]
where \(\epsilon_m\) is the surrogate-error bound from Lemma~\ref{lem:misspec}. Then, with probability at least \(1-c'/T\),
\[
R_T^{\Gamma,C}(\mathcal M)\le \tilde c(d\sqrt{T}+d^{3/2}C),
\]
where \(R_T^{\Gamma,C}(\mathcal M)\) denotes the regret of Algorithm~\ref{alg:unknown} on the uncorrupted instance \(\mathcal M\), when the algorithm selects actions using the corrupted feedback \(r_t^C\). Here \(c'>0\) is a universal constant and \(\tilde c\) absorbs logarithmic factors.
\end{corollary}

\begin{proof}
Condition on the high-probability event of Lemma~\ref{lem:misspec}. In epoch $m$, for every $\theta\in\Theta'$, decompose
\begin{align*}
\left|\langle g^{(m)}(\theta),\hat\theta_m\rangle
-\langle g_\pi(\theta),\theta_\star\rangle\right|
&\le
\left|\langle g^{(m)}(\theta),\hat\theta_m-\theta_\star\rangle\right|
+
\left|\langle g^{(m)}(\theta)-g_\pi(\theta),\theta_\star\rangle\right|.
\end{align*}
The second term is bounded by $\epsilon_m$ by Lemma~\ref{lem:misspec}. For the first term, we use the same PE estimation step as Appendix~B.4 of \citet{hanna2023contexts}. First consider a base learner whose candidate corruption budget covers the true epoch corruption, i.e., $C'_m\ge C_m$. The analogue of Equation~(26) in \citet{hanna2023contexts} controls the non-corrupted PE estimation term, while Lemma~1 of \citet{bogunovic2021} adds the corruption inflation term proportional to $C'_m/H_m$. Applied with total epoch-wise misspecification $\bar\epsilon_m=\epsilon_m+2C_{\rm mix}T^{-c_\tau}$, and combined with the additional surrogate comparison error $\epsilon_m$ from Lemma~\ref{lem:misspec}, this implies uniformly over $\theta\in\Theta'$ that the stated choice of $\gamma_m$ gives $\left|\langle g^{(m)}(\theta),\hat\theta_m\rangle -\langle g_\pi(\theta),\theta_\star\rangle\right|\le \gamma_m/2$.

A standard calculation as in Theorem~1 of \citet{bogunovic2021} gives a regret bound of $\widetilde O(d\sqrt{H_m}+d^{3/2}C_m+\sqrt d\,H_m\bar\epsilon_m)$ with high probability for any valid candidate budget $C'_m\ge C_m$. Then, using the model-selection result of Theorem~4 in \citet{wei2022}, G-COBE adapts to the unknown $C_m$ while maintaining this regret bound up to logarithmic and lower-order terms, yielding $R_{H_m}^{\Lambda}(L_m) \le \widetilde O(d\sqrt{H_m}+d^{3/2}C_m+\sqrt d\,H_m\bar\epsilon_m)$ with high probability, where \(R_{H_m}^{\Lambda}(L_m)\) denotes the regret of the epoch-\(m\) G-COBE subroutine on the reduced linear bandit instance \(L_m\). Applying Lemma~\ref{lem:lambda_regret} to sum over all epochs, using $\sum_{m=1}^M C_m \le C$, and absorbing the cumulative $\sqrt d\,H_m\bar\epsilon_m$ term as in the misspecification analysis, concludes the proof.
\end{proof}

\textbf{Remark.} The corruption model in Corollary~\ref{cor:adversarial-corruption} should not be
confused with the main model of \citet{he2022nearly}. Their algorithm studies
post-action corruption of the selected reward, whereas our corollary uses the pre-action
all-action corruption budget. In the unknown-budget
case, their guarantee becomes sublinear only when the true corruption is below
the chosen estimate, e.g. $C\le \sqrt T$, and otherwise reverts to a linear
$O(T)$ bound; thus it does not directly subsume the unknown-budget corruption
variant considered here.

\section{Details on Experimental Results}
\label{app:experiments}

\subsection{Bandit Instance Construction}
\label{app:realdata-details}

We construct a contextual linear bandit instance from the synchronized vehicle GPS--audio stream. Each timestamp defines one context, and each available action becomes a subset \(S\) of the sensing nodes. We keep \(N=9\) nodes and use the full subset universe with \(1\le |S|\le 3\), giving
\[
|\mathcal A_t|
=
\binom{9}{1}+\binom{9}{2}+\binom{9}{3}
=
129
\]
actions at every timestamp. For a subset \(S\), let
\(d_{(1)}^t(S)\le \cdots \le d_{(|S|)}^t(S)\) denote the sorted GPS-derived distances from the selected nodes to the vehicle at time \(t\). The oracle utility is set as,
\begin{align}
u_t(S)
&=
\sum_{j=1}^{|S|}
\frac{w_j}{1+d_{(j)}^t(S)/\rho},
\qquad
\rho=10,\qquad
(w_1,w_2,w_3)=(1,0.45,0.20).
\nonumber
\end{align}

We then construct feature vectors $c_t\in\mathbb{R}^{85}$ and $\psi_t(S)\in \mathbb{R}^{42}$ for every timestep $t$, where features are grouped as follows and historical RSSI features use the trailing five synchronized samples:
\begin{itemize}
\item \textbf{Context vector.} The context vector \({c_t}\) consists of:
(i) per-sensor features for each sensor: current RSSI, historical RSSI mean, historical RSSI standard deviation, historical RSSI slope, one-step RSSI difference, acoustic energy share, acoustic rank percentile, and normalized sensor coordinates;
and (ii) global acoustic features: acoustic entropy, top-one/top-two energy-share gap, and acoustic center-of-mass coordinates.
\item \textbf{Raw action vector.} The raw subset-action vector \(\psi_t(S)\) has four blocks:
(i) a subset mask \(m(S)\in\{0,1\}^N\), whose \(i\)-th entry indicates whether sensor \(i\in S\);
(ii) node descriptors for up to three selected sensors, where each descriptor contains normalized sensor coordinates, current RSSI, historical RSSI mean, historical RSSI slope, and acoustic energy share;
(iii) static subset geometry, consisting of the three pairwise sensor distances, zero-padded when fewer than three pairs exist, the subset centroid coordinates, the maximum pairwise distance within the subset, and the triangle area, set to zero unless \(|S|=3\);
and (iv) acoustic aggregates over \(S\): subset size, sum/max/min energy share, mean RSSI, RSSI variance, mean RSSI slope, and an indicator for whether the sensor with the largest energy share at time \(t\) belongs to \(S\).
\end{itemize}

The feature vector dimensions and core parameters are summarized in Table~\ref{tab:real-core-feature-params}.

The embedding model is then trained on \(964{,}017\) time--subset examples, split chronologically into \(578{,}307\), \(192{,}726\), and \(192{,}984\) train/validation/test examples to obtain the linear embeddings $z_c(c_t) \in \mathbb{R}^8$ and $z_a(\psi_t(S))\in \mathbb{R}^8$ combined as,
\begin{align}
z_t(S) 
&\coloneqq
\big[z_c(c_t),\,z_a(\psi_t(S)),\,z_c(c_t)\odot z_a(\psi_t(S))\big]
\in\mathbb R^{24}.
\nonumber
\end{align}

\begin{table}[t]
\centering
\caption{Core numerical parameters and feature dimensions for the bandit instance construction.}
\label{tab:real-core-feature-params}
\begingroup
\setlength{\tabcolsep}{3pt}
\renewcommand{\arraystretch}{1.10}
\begin{tabularx}{\linewidth}{@{}l|c|X|c@{}}
\toprule
\textbf{Object} & \textbf{Symbol} & \textbf{Definition} & \textbf{Value} \\
\midrule
\multirow{1}{*}{Nodes}
& \(N\)
& sensing nodes
& \(9\)
\\
\midrule
\multirow{2}{*}{Actions}
& \(|S|\)
& selected subset size
& \(1,2,3\)
\\
& \(|\mathcal A_t|\)
& \(\sum_{k=1}^{3}\binom{N}{k}
 = \binom{9}{1}+\binom{9}{2}+\binom{9}{3}\)
& \(129\)
\\
\midrule
\multirow{2}{*}{Utility}
& \(\rho\)
& distance scale in \(u_t(S)\)
& \(10\)
\\
& \((w_1,w_2,w_3)\)
& weights on sorted selected-node distances
& \((1,0.45,0.20)\)
\\
\midrule
\multirow{4}{*}{Context \(c_t\)}
& \(d_{\mathrm{node}}\)
& features per node
& \(9\)
\\
& \(N d_{\mathrm{node}}\)
& node-feature block
& \(9\cdot 9=81\)
\\
& \(d_{\mathrm{global}}\)
& global acoustic-summary features
& \(4\)
\\
& \(d_c\)
& total context dimension: \(N d_{\mathrm{node}}+d_{\mathrm{global}}\)
& \(85\)
\\
\midrule
\multirow{5}{*}{Action \(\psi_t(S)\)}
& \(d_{\mathrm{mask}}\)
& subset mask
& \(9\)
\\
& \(d_{\mathrm{desc}}\)
& node descriptors: \(3\) nodes \(\times\) \(6\) features
& \(18\)
\\
& \(d_{\mathrm{geom}}\)
& static subset geometry
& \(7\)
\\
& \(d_{\mathrm{agg}}\)
& acoustic aggregate features
& \(8\)
\\
& \(d_{\psi}\)
& total raw subset-action dimension
& \(42\)
\\
\bottomrule
\end{tabularx}
\endgroup
\end{table}

The model architecture is summarized in Table~\ref{tab:real-two-tower-architecture}. The towers are trained as a supervised embedding model for the subset utility \(u_t(S)\). During training, a linear head \(h_\omega(z)=\alpha+z^\top\omega\) maps \(z_t(S)\in\mathbb R^{24}\) to a scalar prediction, and the model is optimized by mean-squared error against \(u_t(S)\). After training, we discard this head, freeze the embeddings, and fit $b\in \mathbb{R},\theta \in\mathbb{R}^{24}$ with a separate ridge regression model,
\[
u_t(S)\approx b+z_t(S)^\top\theta ,
\]
over the frozen embedding matrix, which has one row per time--subset pair and size \(964{,}017\times 24\). The fitted ridge intercept \(b\) and coefficient vector \(\theta\in\mathbb R^{24}\) define the parameter of the final linearized bandit instance:
\begin{align}
x_t(S) &=[1,z_t(S)^\top]^\top\in\mathbb R^{25},
\qquad
\theta_\star=[b,\theta^\top]^\top\in\mathbb R^{25}, \nonumber
\\
r_t^{\varepsilon}(S) &= u_t(S),
\qquad
r_t^{\mathrm{lin}}(S)
= x_t(S)^\top\theta_\star. \nonumber
\end{align}
Thus the final linearized bandit instance has feature dimension \(d=25\); its oracle rewards are the original utilities, and its linearized rewards are the ridge predictions on the learned embeddings.

\begin{table}[t]
\centering
\caption{Two-tower embedding architecture.}
\label{tab:real-two-tower-architecture}
\begingroup
\setlength{\tabcolsep}{4pt}
\renewcommand{\arraystretch}{1.12}
\begin{tabularx}{\linewidth}{@{}l|c|c|X|c@{}}
\toprule
\textbf{Tower} & \textbf{Input} & \textbf{Hidden widths} & \textbf{Nonlinear layers} & \textbf{Output} \\
\midrule
Context tower & \(c_t\in\mathbb R^{85}\) & \(512,256\) &
LayerNorm after the first hidden layer; ReLU after both hidden layers
& \(z_c(c_t)\in\mathbb R^{8}\) \\
Action tower & \(\psi_t(S)\in\mathbb R^{42}\) & \(512,256\) &
LayerNorm after the first hidden layer; ReLU after both hidden layers
& \(z_a(\psi_t(S))\in\mathbb R^{8}\) \\
\bottomrule
\end{tabularx}
\endgroup
\end{table}

 To obtain a finite-state proxy for the context process, each time step is represented by its vectorized per-time action-feature matrix, reduced by PCA, and clustered by \(K\)-means; the resulting empirical state sequence is used to estimate the mixing proxy. Figure~\ref{fig:cyclic_instance_diagnostics} illustrates the results of this process and the final reward distributions. Subsequently for the long-horizon instance, we replay the constructed real-data sequence such a subset is selected every \(5\) seconds, producing \(10^6\) replayed contexts from an original sequence of \(7208\) contexts.

\begin{table}[t]
\centering
\caption{Environment characteristics for the bandit instance}
\label{tab:real-replay-markov}
\begingroup
\setlength{\tabcolsep}{3pt}
\renewcommand{\arraystretch}{1.10}
\begin{tabularx}{\linewidth}{@{}l|c|X|c@{}}
\toprule
\textbf{Object} & \textbf{Symbol} & \textbf{Construction} & \textbf{Value} \\
\midrule
\multirow{3}{*}{Replay}
& \(T_{\mathrm{orig}}\)
& original real-data contexts
& \(7208\)
\\
& \(T_{\mathrm{replay}}\)
& replay horizon
& \(10^6\)
\\
& \(\Delta_{\mathrm{replay}}\)
& effective time step between replayed contexts
& \(5\) sec.
\\
\midrule
\multirow{2}{*}{State proxy}
& \(d_{\mathrm{PCA}}\)
& PCA dimension before clustering
& \(10\)
\\
& \(K_{\mathrm{state}}\)
& number of \(K\)-means clusters after PCA
& \(20\)
\\
\midrule
\multirow{2}{*}{Mixing proxy}
& \(\widehat C_{\mathrm{mix}}\)
& fitted prefactor in \eqref{eq:UGE}
& \(1.29\)
\\
& \(\hat\beta\)
& fitted geometric rate from the empirical finite-state proxy
& \(0.808\)
\\
\midrule
\multirow{2}{*}{Misspecification}
& \(|\zeta_t|_{0.99}\)
& \(99\%\) quantile of model misspecification over contexts
& \(0.1247\)
\\
& \(\|\epsilon\|_{\infty}\)
& maximum of model misspecification over contexts
& \(0.3160\)
\\
\midrule
\multirow{2}{*}{Reward noise}
& \(\sigma\)
& additive Gaussian noise std in linearized-reward evaluation
& \(1.0\)
\\
& --
& additive noise std in actual-utility evaluation
& \(0.0\)
\\
\bottomrule
\end{tabularx}
\endgroup
\end{table}

\begin{figure}
\centering

\begin{subfigure}{0.32\textwidth}
    \centering
    \includegraphics[width=\linewidth]
    {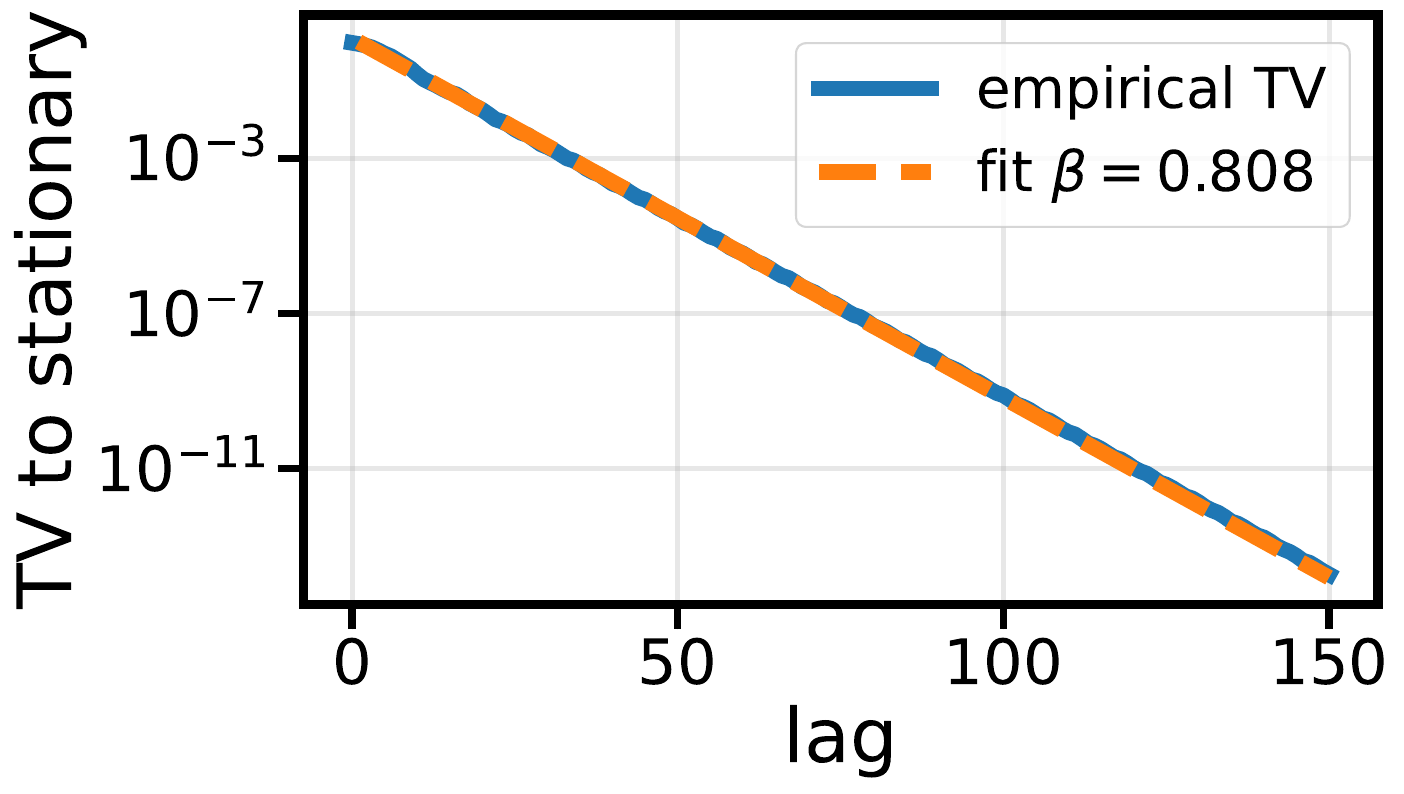}
    \caption{TV mixing curve}
    \label{fig:tv_mixing_curve}
\end{subfigure}
\hfill
\begin{subfigure}{0.32\textwidth}
    \centering
    \includegraphics[width=\linewidth]
    {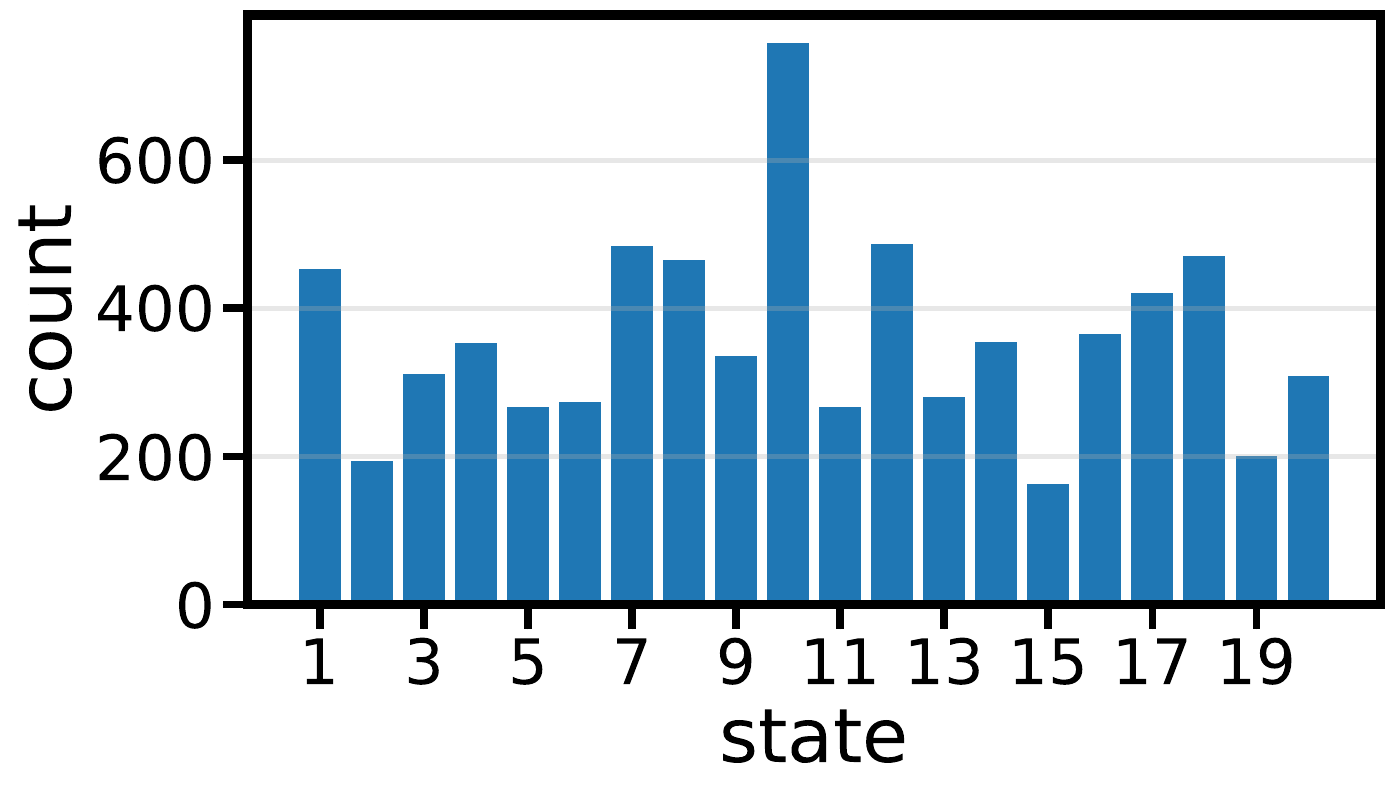}
    \caption{State occupancy}
    \label{fig:state_occupancy}
\end{subfigure}
\hfill
\begin{subfigure}{0.32\textwidth}
    \centering
    \includegraphics[width=\linewidth]
    {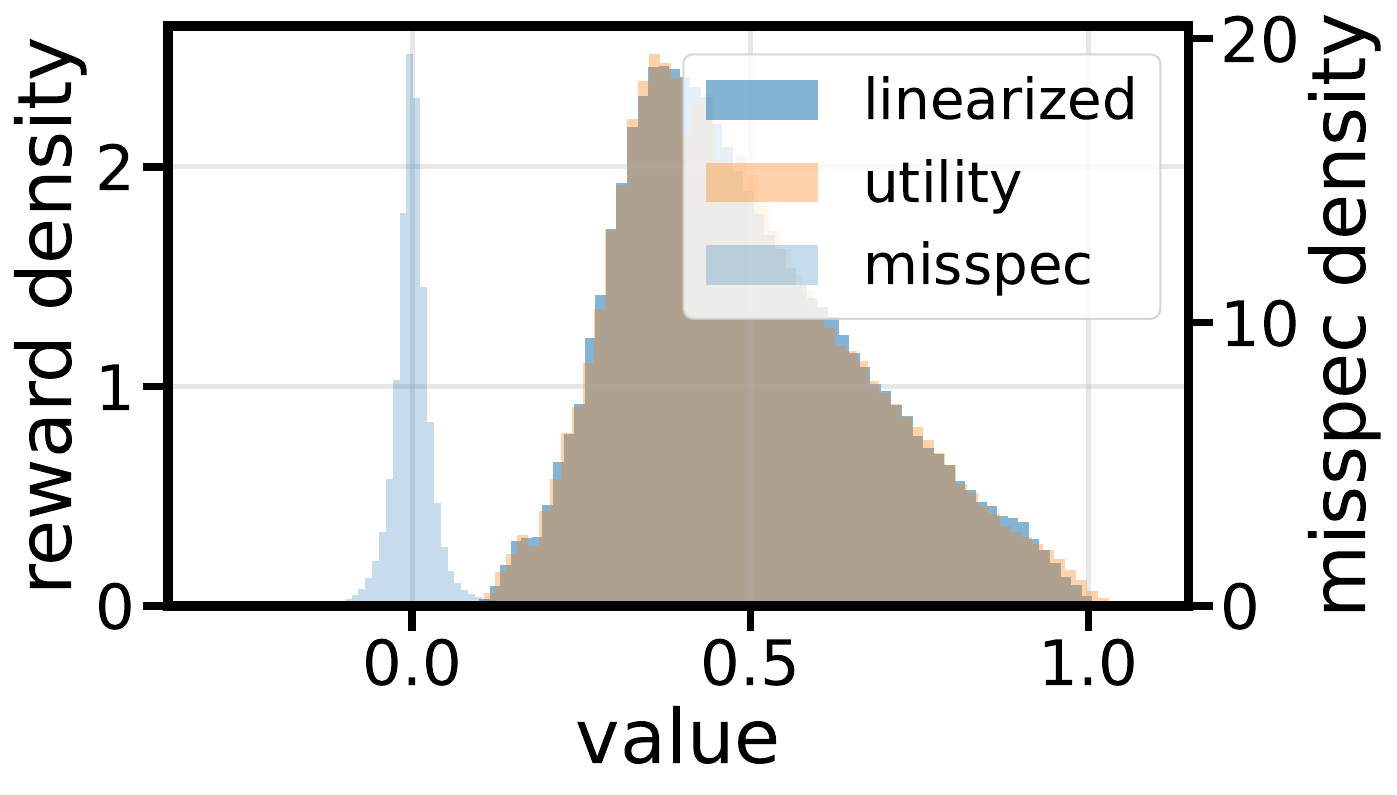}
    \caption{Reward distributions}
    \label{fig:reward_misspec_distribution}
\end{subfigure}

\caption{
Diagnostics for the Markovian bandit instance over one replay period. (a) demonstrates ergodicity, (b) shows the visitation frequency of context groups and (c) demonstrates the rewards obtained from the model and model misspecification
}

\label{fig:cyclic_instance_diagnostics}
\end{figure}

\subsection{Practical Implementation of Algorithms}
\label{app:prac_alg}

In the experiments, Algorithm~\ref{alg:known} uses a LinUCB-style linear bandit subroutine $\Lambda$ for the reduced single-context instance. For Algorithm~\ref{alg:unknown} we use a LinUCB instance as $\Lambda_{\epsilon_m}$ rather than the theoretically supported PE algorithm, which is computationally unfeasible for our case. This creates a divergence from the theoretically established regret guarantee but works well in practice. For the arbitrary action selection periods for Algorithms~\ref{alg:known} and~\ref{alg:unknown} we use a $\tau=0$ version of the corresponding algorithms. In the main algorithm periods, the selected action-reward pairs by these versions are supplied by delayed feedback, which is consistent with our theory. For Algorithm~\ref{alg:unknown}, we use an epoch schedule that grows with $100^m$ instead of the dyadic $2^m$ and a finite bank of surrogate directions instead of explicitly enumerating a \(1/T\)-net. We note that the arbitrary choice of this radix instead of $2$ is supported within our theorems as long as it is constant and finite. Surrogate rows are normalized before LinUCB updates, and the UCB bonus is clipped in the reduced algorithms for numerical stability. The choice of the delay coefficient $c_\tau = 1.0$ provides sufficient bias reduction, while keeping the extra regret from the additive $\tau$ terms minimal. The baseline we compare to is direct LinUCB applied to the observed action set \(\mathcal A_t\), without surrogate actions or delayed feedback. 

We summarize these selections and provide additional hyperparameters in Table~\ref{tab:practical-alg-params}.

\begin{table}[t]
\centering
\caption{Practical implementation parameters for the tested algorithms.}
\label{tab:practical-alg-params}
\begingroup
\setlength{\tabcolsep}{3pt}
\renewcommand{\arraystretch}{1.10}
\begin{tabularx}{\linewidth}{@{}l|c|X|c@{}}
\toprule
\textbf{Object} & \textbf{Symbol} & \textbf{Role} & \textbf{Value} \\
\midrule
\multirow{2}{*}{Instance}
& \(|\mathcal A_t|\)
& number of available actions per context
& \(129\)
\\
& \(|\Theta|\)
& finite bank size of surrogate directions
& \(256\)
\\
\midrule
\multirow{2}{*}{Delay}
& \(c_\tau\)
& delayed-feedback setting
& \(1\)
\\
& \(c_\tau\)
& no-delay ablation
& \(0\)
\\
\midrule
Algorithm~\ref{alg:unknown}
& \(\mathrm{radix}\)
& epoch growth factor
& \(100\)
\\
\midrule
\multirow{3}{*}{LinUCB}
& \(\lambda\)
& regularization, used by baseline and reduced subroutines
& \(1.0\)
\\
& \(\alpha\)
& UCB exploration bonus coefficient, used by baseline and reduced subroutines
& \(2.0\)
\\
& \(B\)
& bonus cap, used only by the reduced subroutines
& \(1.0\)
\\
\bottomrule
\end{tabularx}
\endgroup
\end{table}

\subsection{Simulation of Algorithm~\ref{alg:known} Under Linearized Rewards}
\label{app:extra_sim}

The associated plot is presented in Figure~\ref{fig:alg1_tau_comparison}.

\begin{figure}
\centering
\begin{subfigure}{0.48\linewidth}
    \centering
    \includegraphics[width=\linewidth]{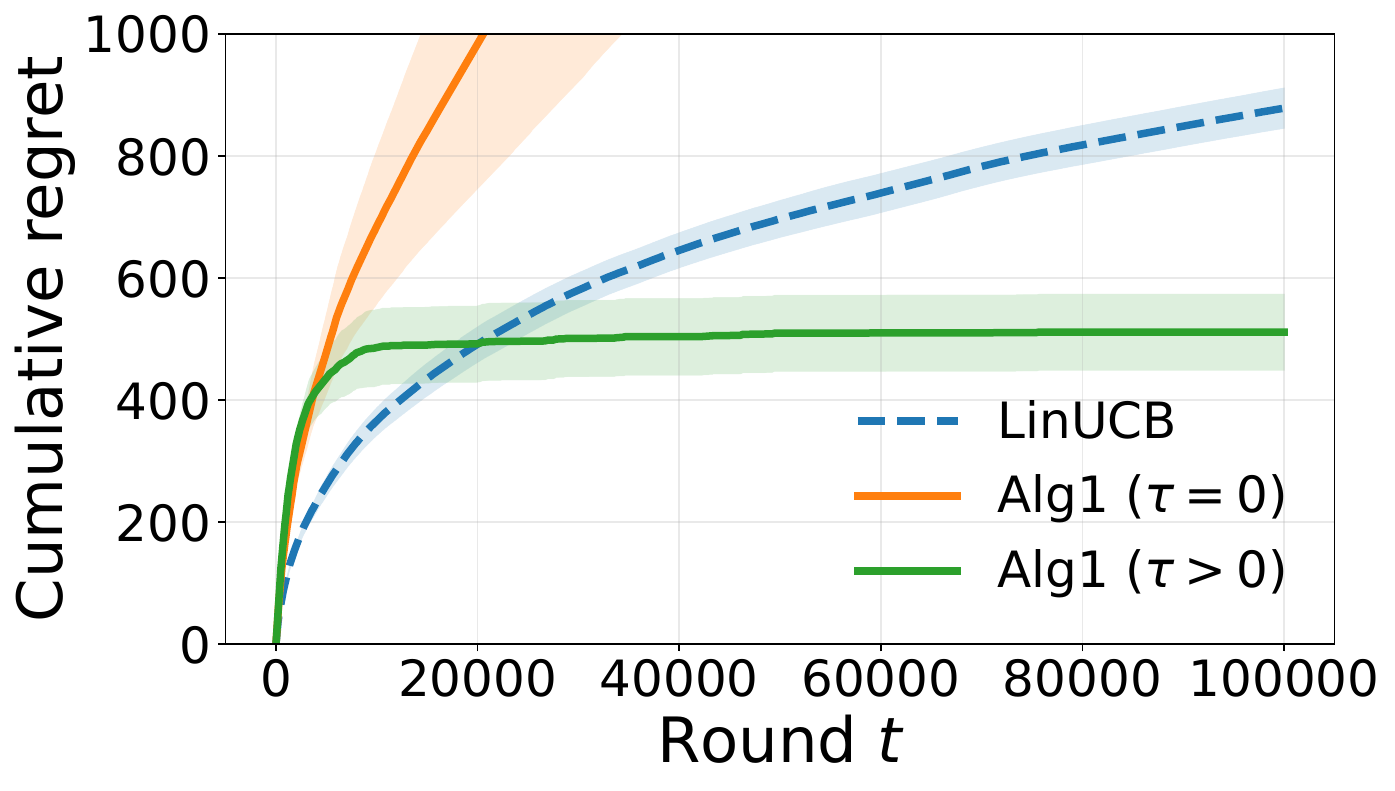}
    \caption{Cumulative regret}
\end{subfigure}
\hfill
\begin{subfigure}{0.48\linewidth}
    \centering
    \includegraphics[width=\linewidth]{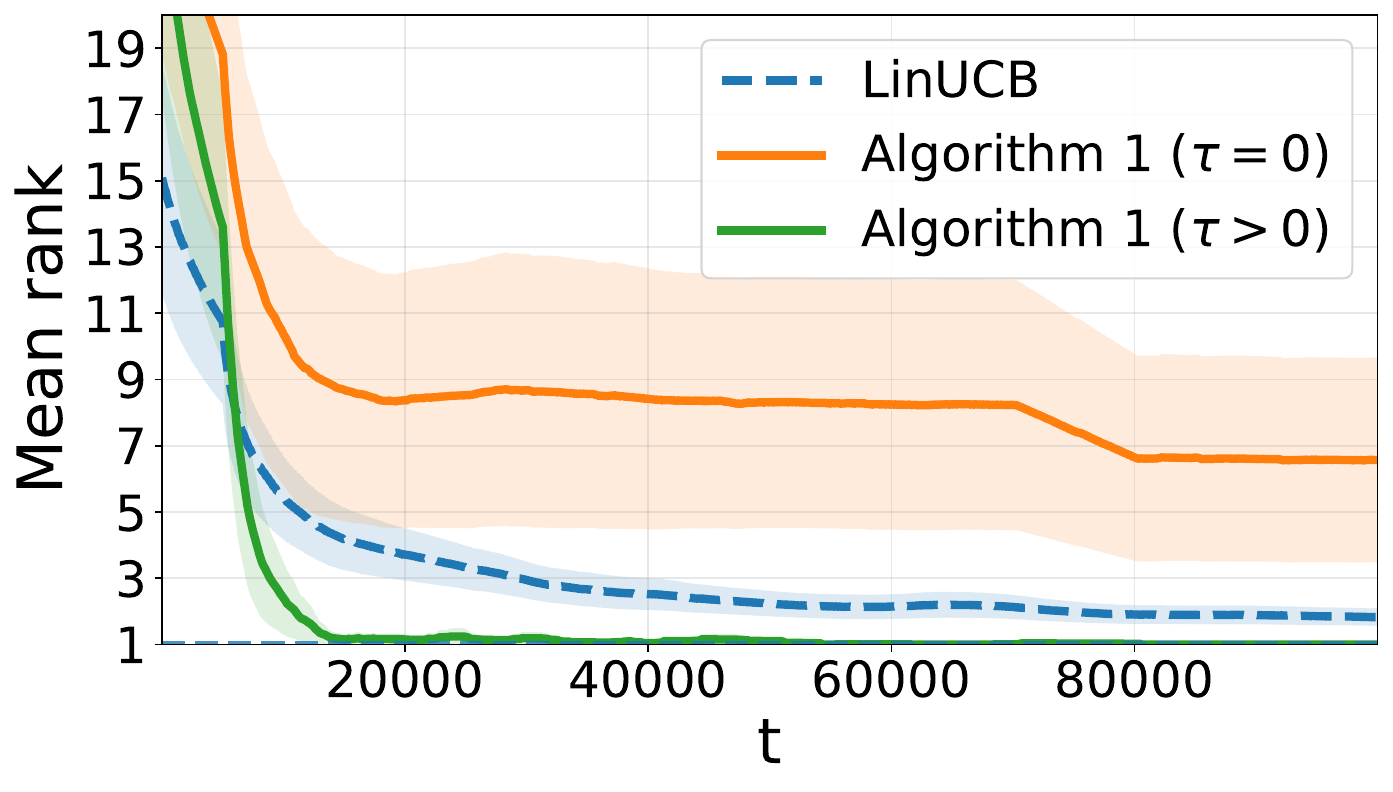}
    \caption{Mean rank of selected action, $|\mathcal A_t|=129$}
\end{subfigure}

\caption{Performance of Algorithm~\ref{alg:known} on the bandit instance with linearized rewards. Shaded regions show \(\pm\) one standard error around the mean across runs.}
\label{fig:alg1_tau_comparison}
\end{figure}

\section{Broader Impacts}
\label{app:broader-impacts}

This work is primarily theoretical, studying regret guarantees for linear bandits with Markovian context processes. Its positive societal impact is in improving sequential decision-making methods for systems with temporally correlated availability, especially resource-constrained sensing systems where one must decide which sensors, actions, or interventions to activate. In such settings, algorithms that preserve strong regret guarantees can improve sample efficiency, reduce unnecessary sensing or computation, and support more efficient deployment of networked sensing infrastructure.

The main potential negative impacts arise from applications of sequential sensor selection to tracking, monitoring, and surveillance. Although our experiments are based on a vehicle-tracking sensing task, similar methods could be used in settings where the tracked entity is a person, device, or group, raising privacy and civil-liberties concerns.


\newpage
\input{checklist.tex}

\end{document}

%% file: checklist.tex
\section*{NeurIPS Paper Checklist}


\begin{enumerate}

\item {\bf Claims}
    \item[] Question: Do the main claims made in the abstract and introduction accurately reflect the paper's contributions and scope?
    \item[] Answer: \answerYes{}
    \item[] Justification:  The abstract and introduction state the paper’s main contribution as a reduction from context-Markovian linear bandits to ordinary linear bandits using stationary surrogate actions and delayed feedback under uniform geometric ergodicity. They distinguish the known- and unknown-stationary-distribution regimes and match the theoretical scope of the results, namely high-probability regret guarantees that recover linear-bandit rates in sufficiently fast-mixing regimes or up to explicit mixing-dependent factors. As suggested, the empirical claims are appropriately limited to semi-synthetic and real-data vehicle-tracking instances and are framed as validation against LinUCB. The main assumptions and limitations, including exogenous Markovian context evolution, mixing requirements, and computational efficiency, are reflected in the problem setup and theorem statements.

\item {\bf Limitations}
    \item[] Question: Does the paper discuss the limitations of the work performed by the authors?
    \item[] Answer: \answerYes{}
    \item[] Justification: The paper discusses the main limitations of the approach, especially the requirement that the Markovian context process be ergodic and mix sufficiently fast for the regret guarantees to be comparable to ordinary linear-bandit rates. It also discusses computational limitations when the unknown parameter space is unstructured, since efficient implementation then relies on standard optimization and regression oracle assumptions used in the bandit literature. The empirical evaluation is limited to the constructed semi-synthetic and real-data vehicle-tracking instances, and performance can depend on the learned embeddings and the validity of the linearized reward model.


\item {\bf Theory assumptions and proofs}
    \item[] Question: For each theoretical result, does the paper provide the full set of assumptions and a complete (and correct) proof?
    \item[] Answer: \answerYes{} 
    \item[] Justification: We provide the full proof of each theorem and lemma in the corresponding appendix sections, with proof sketches provided in the main paper.

    \item {\bf Experimental result reproducibility}
    \item[] Question: Does the paper fully disclose all the information needed to reproduce the main experimental results of the paper to the extent that it affects the main claims and/or conclusions of the paper (regardless of whether the code and data are provided or not)?
    \item[] Answer: \answerYes{} 
    \item[] Justification: We provide the associated code and data in the supplementary material, where the provided code is complete to reproduce the experimental results. We also specify in Appendix~\ref{app:experiments} what parameters are used and estimated for the algorithm and the environment.

\item {\bf Open access to data and code}
    \item[] Question: Does the paper provide open access to the data and code, with sufficient instructions to faithfully reproduce the main experimental results, as described in supplemental material?
    \item[] Answer: \answerYes{} 
    \item[] Justification: We provide the associated code and data in the supplementary material, where the provided code is complete to reproduce the experimental results. We also make the dataset publicly available at the time of publication.

\item {\bf Experimental setting/details}
    \item[] Question: Does the paper specify all the training and test details (e.g., data splits, hyperparameters, how they were chosen, type of optimizer) necessary to understand the results?
    \item[] Answer: \answerYes{} 
    \item[] Justification: The required details are given in Appendix~\ref{app:experiments}. 

\item {\bf Experiment statistical significance}
    \item[] Question: Does the paper report error bars suitably and correctly defined or other appropriate information about the statistical significance of the experiments?
    \item[] Answer: \answerYes{} 
    \item[] Justification: Figures \ref{fig:alg2_tau_comparison} and~\ref{fig:alg1_tau_comparison} feature shaded regions for \(\pm\) one standard error around the mean across runs. Figure~\ref{fig:alg1-alg2-real-comparison} illustrates a run over the whole test dataset.

\item {\bf Experiments compute resources}
    \item[] Question: For each experiment, does the paper provide sufficient information on the computer resources (type of compute workers, memory, time of execution) needed to reproduce the experiments?
    \item[] Answer: \answerNo{} 
    \item[] Justification: The current version does not report detailed hardware, memory, or runtime information since the experiments do not require substantial compute or memory.
    
\item {\bf Code of ethics}
    \item[] Question: Does the research conducted in the paper conform, in every respect, with the NeurIPS Code of Ethics \url{https://neurips.cc/public/EthicsGuidelines}?
    \item[] Answer: \answerYes{} 
    \item[] Justification: We confirm that our research conforms to NeurIPS Code of Ethics to the best of our knowledge.


\item {\bf Broader impacts}
    \item[] Question: Does the paper discuss both potential positive societal impacts and negative societal impacts of the work performed?
    \item[] Answer: \answerYes{} 
    \item[] Justification: We address potential positive and negative broader societal impacts in Appendix~\ref{app:broader-impacts}.
    
\item {\bf Safeguards}
    \item[] Question: Does the paper describe safeguards that have been put in place for responsible release of data or models that have a high risk for misuse (e.g., pre-trained language models, image generators, or scraped datasets)?
    \item[] Answer: \answerNA{} 
    \item[] Justification: We do not believe the released data contains risks for misuse, after obtaining the necessary permission from the institution that hosted the experiments.

\item {\bf Licenses for existing assets}
    \item[] Question: Are the creators or original owners of assets (e.g., code, data, models), used in the paper, properly credited and are the license and terms of use explicitly mentioned and properly respected?
    \item[] Answer: \answerNA{} 
    \item[] Justification: The paper does not use existing third-party datasets, pretrained models, or research assets as central experimental assets; the main dataset was collected by the authors.

\item {\bf New assets}
    \item[] Question: Are new assets introduced in the paper well documented and is the documentation provided alongside the assets?
    \item[] Answer: \answerYes{} 
    \item[] Justification: We provide a verbal description on how the dataset was obtained in Section~\ref{sec:app} as well as releasing the dataset. The documentation for the code is given as inline comments.

\item {\bf Crowdsourcing and research with human subjects}
    \item[] Question: For crowdsourcing experiments and research with human subjects, does the paper include the full text of instructions given to participants and screenshots, if applicable, as well as details about compensation (if any)? 
    \item[] Answer: \answerNA{} 
    \item[] Justification: Our paper does not involve crowdsourcing nor research with human subjects.

\item {\bf Institutional review board (IRB) approvals or equivalent for research with human subjects}
    \item[] Question: Does the paper describe potential risks incurred by study participants, whether such risks were disclosed to the subjects, and whether Institutional Review Board (IRB) approvals (or an equivalent approval/review based on the requirements of your country or institution) were obtained?
    \item[] Answer: \answerNA{} 
    \item[] Justification: Our paper does not involve crowdsourcing nor research with human subjects.

\item {\bf Declaration of LLM usage}
    \item[] Question: Does the paper describe the usage of LLMs if it is an important, original, or non-standard component of the core methods in this research? Note that if the LLM is used only for writing, editing, or formatting purposes and does \emph{not} impact the core methodology, scientific rigor, or originality of the research, declaration is not required.
    \item[] Answer: \answerNA{} 
    \item[] Justification: LLMs are only used for editing and formatting purposes.

\end{enumerate}

%% file: references.bib
@misc{ang2023evaluating,
      title={Evaluating Online Bandit Exploration In Large-Scale Recommender System}, 
      author={Hongbo Guo and Ruben Naeff and Alex Nikulkov and Zheqing Zhu},
      year={2023},
      eprint={2304.02572},
      archivePrefix={arXiv},
      primaryClass={cs.IR},
      url={https://arxiv.org/abs/2304.02572}, 
}

@article{bojinov2022Online,
	author = {Bojinov, Iavor and Gupta, Somit},
	journal = {Harvard Data Science Review},
	number = {3},
	year = {2022},
	month = {jul 28},
	note = {https://hdsr.mitpress.mit.edu/pub/aj31wj81},
	publisher = {The MIT Press},
	title = {Online {Experimentation}: Benefits, {Operational} and {Methodological} {Challenges}, and {Scaling} {Guide}},
	volume = {4},
}

@inproceedings{dani2008stochastic,
  author    = {Varsha Dani and Thomas P. Hayes and Sham M. Kakade},
  title     = {Stochastic Linear Optimization under Bandit Feedback},
  booktitle = {Proceedings of the 21st Annual Conference on Learning Theory (COLT)},
  editor    = {Rocco A. Servedio and Tong Zhang},
  pages     = {355--366},
  year      = {2008}
}

@inproceedings{abbasi2011improved,
  author    = {Yasin Abbasi{-}Yadkori and D{\'a}vid P{\'a}l and Csaba Szepesv{\'a}ri},
  title     = {Improved Algorithms for Linear Stochastic Bandits},
  booktitle = {Advances in Neural Information Processing Systems},
  volume    = {24},
  pages     = {2312--2320},
  year      = {2011},
  url       = {http://papers.nips.cc/paper\_files/paper/2011/hash/e1d5be1c7f2f456670de3d53c7b54f4a-Abstract.html}
}

@book{lattimore2020bandit,
  author    = {Tor Lattimore and Csaba Szepesv{\'a}ri},
  title     = {Bandit Algorithms},
  publisher = {Cambridge University Press},
  year      = {2020},
  isbn      = {9781108486828},
  url       = {https://banditalgs.com}
}

@inproceedings{chu2011contextual,
  title     = {Contextual Bandits with Linear Payoff Functions},
  author    = {Wei Chu and Lihong Li and Lev Reyzin and Robert Schapire},
  booktitle = {Proceedings of the Fourteenth International Conference on Artificial Intelligence and Statistics},
  pages     = {208--214},
  year      = {2011},
  editor    = {Geoffrey Gordon and David Dunson and Miroslav Dud{\'\i}k},
  volume    = {15},
  series    = {Proceedings of Machine Learning Research},
  publisher = {PMLR},
  url       = {https://proceedings.mlr.press/v15/chu11a.html}
}

@inproceedings{agrawal2013thompson,
  title     = {Thompson Sampling for Contextual Bandits with Linear Payoffs},
  author    = {Shipra Agrawal and Navin Goyal},
  booktitle = {Proceedings of the 30th International Conference on Machine Learning},
  series    = {Proceedings of Machine Learning Research},
  volume    = {28},
  number    = {3},
  pages     = {127--135},
  year      = {2013},
  publisher = {PMLR},
  url       = {https://proceedings.mlr.press/v28/agrawal13.html}
}

@InProceedings{hao2020adaptive,
  title = 	 {Adaptive Exploration in Linear Contextual Bandit},
  author =       {Hao, Botao and Lattimore, Tor and Szepesvari, Csaba},
  booktitle = 	 {Proceedings of the Twenty Third International Conference on Artificial Intelligence and Statistics},
  pages = 	 {3536--3545},
  year = 	 {2020},
  editor = 	 {Chiappa, Silvia and Calandra, Roberto},
  volume = 	 {108},
  series = 	 {Proceedings of Machine Learning Research},
  month = 	 {26--28 Aug},
  publisher =    {PMLR},
  url = 	 {https://proceedings.mlr.press/v108/hao20b.html}}

@inproceedings{tirinzoni2020asymptotically,
  title     = {An Asymptotically Optimal Primal-Dual Incremental Algorithm for Contextual Linear Bandits},
  author    = {Andrea Tirinzoni and Matteo Pirotta and Marcello Restelli and Alessandro Lazaric},
  booktitle = {Advances in Neural Information Processing Systems},
  volume    = {33},
  year      = {2020}
}

@InProceedings{hanna2023contexts,
  title = 	 {Contexts can be Cheap: Solving Stochastic Contextual Bandits with Linear Bandit Algorithms},
  author =       {Hanna, Osama A and Yang, Lin and Fragouli, Christina},
  booktitle = 	 {Proceedings of Thirty Sixth Conference on Learning Theory},
  pages = 	 {1791--1821},
  year = 	 {2023},
  editor = 	 {Neu, Gergely and Rosasco, Lorenzo},
  volume = 	 {195},
  series = 	 {Proceedings of Machine Learning Research},
  month = 	 {12--15 Jul},
  publisher =    {PMLR},
  url = 	 {https://proceedings.mlr.press/v195/hanna23a.html},
}

@inproceedings{neu2014online,
  title     = {Online Combinatorial Optimization with Stochastic Decision Sets and Adversarial Losses},
  author    = {Gergely Neu and Michal Valko},
  booktitle = {Advances in Neural Information Processing Systems},
  volume    = {27},
  year      = {2014},
  url       = {https://papers.nips.cc/paper/5381-online-combinatorial-optimization-with-stochastic-decision-sets-and-adversarial-losses}
}

@inproceedings{liu2023bypassing,
  author    = {Haolin Liu and Chen{-}Yu Wei and Julian Zimmert},
  title     = {Bypassing the Simulator: Near-Optimal Adversarial Linear Contextual Bandits},
  booktitle = {Advances in Neural Information Processing Systems},
  volume    = {36},
  year      = {2023},
  url       = {http://papers.nips.cc/paper\_files/paper/2023/hash/a3a661eb3308d0bb686f6a4bac521032-Abstract-Conference.html}
}

@inproceedings{olkhovskaya2023firstsecond,
  author    = {Julia Olkhovskaya and Jack J. Mayo and Tim van Erven and Gergely Neu and Chen{-}Yu Wei},
  title     = {First- and Second-Order Bounds for Adversarial Linear Contextual Bandits},
  booktitle = {Advances in Neural Information Processing Systems},
  volume    = {36},
  year      = {2023},
  url       = {http://papers.nips.cc/paper\_files/paper/2023/hash/c2201e444d2b22a10ca50116a522b9a9-Abstract-Conference.html}
}

@inproceedings{vanerven2025improved,
title={An Improved Algorithm for Adversarial Linear Contextual Bandits via Reduction},
author={Tim van Erven and Jack Mayo and Julia Olkhovskaya and Chen-Yu Wei},
booktitle={The Thirty-ninth Annual Conference on Neural Information Processing Systems},
year={2025},
url={https://openreview.net/forum?id=MfhFiU28hv}
}

@article{tekin2012online,
  title     = {Online Learning of Rested and Restless Bandits},
  author    = {Cem Tekin and Mingyan Liu},
  journal   = {IEEE Transactions on Information Theory},
  volume    = {58},
  number    = {8},
  pages     = {5588--5611},
  year      = {2012},
  publisher = {IEEE}
}

@inproceedings{ortner2012regret,
author = {Ortner, Ronald and Ryabko, Daniil and Auer, Peter and Munos, R{\'e}mi},
title = {Regret bounds for restless markov bandits},
year = {2012},
isbn = {9783642341052},
publisher = {Springer-Verlag},
address = {Berlin, Heidelberg},
url = {https://doi.org/10.1007/978-3-642-34106-9_19},
doi = {10.1007/978-3-642-34106-9_19},
booktitle = {Proceedings of the 23rd International Conference on Algorithmic Learning Theory},
pages = {214–228},
numpages = {15},
location = {Lyon, France},
series = {ALT'12}
}

@article{weber1990index,
  title     = {On an Index Policy for Restless Bandits},
  author    = {Richard R. Weber and Gideon Weiss},
  journal   = {Journal of Applied Probability},
  volume    = {27},
  number    = {3},
  pages     = {637--648},
  year      = {1990},
  publisher = {Cambridge University Press}
}

@article{bertsimas2000restless,
  title     = {Restless Bandits, Linear Programming Relaxations, and a Primal-Dual Index Heuristic},
  author    = {Dimitris Bertsimas and Jos{\'e} Ni{\~n}o{-}Mora},
  journal   = {Operations Research},
  volume    = {48},
  number    = {1},
  pages     = {80--90},
  year      = {2000},
  publisher = {INFORMS}
}

@inproceedings{ok2018exploration,
author = {Ok, Jungseul and Proutiere, Alexandre and Tranos, Damianos},
title = {Exploration in structured reinforcement learning},
year = {2018},
publisher = {Curran Associates Inc.},
address = {Red Hook, NY, USA},
booktitle = {Proceedings of the 32nd International Conference on Neural Information Processing Systems},
pages = {8888–8896},
numpages = {9},
location = {Montr{\'e}al, Canada},
series = {NIPS'18}
}

@inproceedings{nelson2022linearizing,
  title     = {Linearizing Contextual Bandits with Latent State Dynamics},
  author    = {Elliot Nelson and Debarun Bhattacharjya and Tian Gao and Miao Liu and Djallel Bouneffouf and Pascal Poupart},
  booktitle = {Proceedings of the Thirty-Eighth Conference on Uncertainty in Artificial Intelligence},
  series    = {Proceedings of Machine Learning Research},
  volume    = {180},
  pages     = {1477--1487},
  year      = {2022},
  publisher = {PMLR},
  url       = {https://proceedings.mlr.press/v180/nelson22a.html}
}

@article{zeng2024partially,
  title         = {Partially Observable Contextual Bandits with Linear Payoffs},
  author        = {Sihan Zeng and Sujay Bhatt and Alec Koppel and Sumitra Ganesh},
  journal       = {arXiv preprint arXiv:2409.11521},
  year          = {2024},
  url           = {https://arxiv.org/abs/2409.11521},
  archivePrefix = {arXiv},
  eprint        = {2409.11521}
}

@inproceedings{luo2018efficient,
  title     = {Efficient Contextual Bandits in Non-stationary Worlds},
  author    = {Haipeng Luo and Chen{-}Yu Wei and Alekh Agarwal and John Langford},
  booktitle = {Proceedings of The 31st Conference on Learning Theory},
  series    = {Proceedings of Machine Learning Research},
  volume    = {75},
  year      = {2018},
  publisher = {PMLR},
  url       = {https://proceedings.mlr.press/v75/luo18a.html}
}

@inproceedings{chacun2024dronebandit,
author = {Chacun, Guillaume and Akeddar, Mehdi and Rieder, Thomas and Da Rocha Carvalho, Bruno and Zapater, Marina},
title = {DroneBandit: Multi-armed contextual bandits for collaborative edge-to-cloud inference in resource-constrained nanodrones},
year = {2024},
isbn = {9798400706059},
publisher = {Association for Computing Machinery},
address = {New York, NY, USA},
url = {https://doi.org/10.1145/3649476.3658720},
doi = {10.1145/3649476.3658720},
booktitle = {Proceedings of the Great Lakes Symposium on VLSI 2024},
pages = {98–104},
numpages = {7},
location = {Clearwater, FL, USA},
series = {GLSVLSI '24}
}

@ARTICLE{wakayama2023observation,
  author={Wakayama, Shohei and Ahmed, Nisar},
  journal={IEEE Robotics and Automation Letters}, 
  title={Observation-Augmented Contextual Multi-Armed Bandits for Robotic Search and Exploration}, 
  year={2024},
  volume={9},
  number={10},
  pages={8531-8538},
  doi={10.1109/LRA.2024.3448133}}

@Article{medical,
AUTHOR = {Varatharajah, Yogatheesan and Berry, Brent},
TITLE = {A Contextual-Bandit-Based Approach for Informed Decision-Making in Clinical Trials},
JOURNAL = {Life},
VOLUME = {12},
YEAR = {2022},
NUMBER = {8},
ARTICLE-NUMBER = {1277},
URL = {https://www.mdpi.com/2075-1729/12/8/1277},
PubMedID = {36013456},
ISSN = {2075-1729},
DOI = {10.3390/life12081277}
}

@INPROCEEDINGS{yemini2020restless,
  author={Yemini, Michal and Leshem, Amir and Somekh-Baruch, Anelia},
  booktitle={2020 59th IEEE Conference on Decision and Control (CDC)}, 
  title={Restless Hidden Markov Bandit with Linear Rewards}, 
  year={2020},
  volume={},
  number={},
  pages={1183-1189},
  doi={10.1109/CDC42340.2020.9304511}}

@InProceedings{lattimore2020,
  title = 	 {Learning with Good Feature Representations in Bandits and in {RL} with a Generative Model},
  author =       {Lattimore, Tor and Szepesvari, Csaba and Weisz, Gellert},
  booktitle = 	 {Proceedings of the 37th International Conference on Machine Learning},
  pages = 	 {5662--5670},
  year = 	 {2020},
  editor = 	 {III, Hal Daumé and Singh, Aarti},
  volume = 	 {119},
  series = 	 {Proceedings of Machine Learning Research},
  month = 	 {13--18 Jul},
  publisher =    {PMLR},
  url = 	 {https://proceedings.mlr.press/v119/lattimore20a.html},
}

@article{paulin_bound,
title = "Concentration inequalities for Markov chains by Marton couplings and spectral methods",
author = "Daniel Paulin",
year = "2015",
month = sep,
doi = "10.1214/EJP.v20-4039",
language = "English",
volume = "20",
journal = "Electronic journal of probability",
issn = "1083-6489",
publisher = "Institute of Mathematical Statistics",
}

@InProceedings{jin2020provably,
  title = 	 {Provably efficient reinforcement learning with linear function approximation},
  author =       {Jin, Chi and Yang, Zhuoran and Wang, Zhaoran and Jordan, Michael I},
  booktitle = 	 {Proceedings of Thirty Third Conference on Learning Theory},
  pages = 	 {2137--2143},
  year = 	 {2020},
  editor = 	 {Abernethy, Jacob and Agarwal, Shivani},
  volume = 	 {125},
  series = 	 {Proceedings of Machine Learning Research},
  month = 	 {09--12 Jul},
  publisher =    {PMLR},
  url = 	 {https://proceedings.mlr.press/v125/jin20a.html}
}

@inproceedings{foster2020,
author = {Foster, Dylan J. and Gentile, Claudio and Mohri, Mehryar and Zimmert, Julian},
title = {Adapting to misspecification in contextual bandits},
year = {2020},
isbn = {9781713829546},
publisher = {Curran Associates Inc.},
address = {Red Hook, NY, USA},
booktitle = {Proceedings of the 34th International Conference on Neural Information Processing Systems},
articleno = {963},
numpages = {12},
location = {Vancouver, BC, Canada},
series = {NIPS '20}
}

@InProceedings{wei2022,
  title = 	 {A Model Selection Approach for Corruption Robust Reinforcement Learning},
  author =       {Wei, Chen-Yu and Dann, Christoph and Zimmert, Julian},
  booktitle = 	 {Proceedings of The 33rd International Conference on Algorithmic Learning Theory},
  pages = 	 {1043--1096},
  year = 	 {2022},
  editor = 	 {Dasgupta, Sanjoy and Haghtalab, Nika},
  volume = 	 {167},
  series = 	 {Proceedings of Machine Learning Research},
  month = 	 {29 Mar--01 Apr},
  publisher =    {PMLR},
  url = 	 {https://proceedings.mlr.press/v167/wei22a.html},
}

@article{hanna2023efficient,
  title={Efficient batched algorithm for contextual linear bandits with large action space via soft elimination},
  author={Hanna, Osama and Yang, Lin and Fragouli, Christina},
  journal={Advances in Neural Information Processing Systems},
  volume={36},
  pages={56772--56783},
  year={2023}
}

@article{He2022Collaborative,
  author = {He, Shibo and Shi, Kun and Liu, Chen and Guo, Bicheng and Chen, Jiming and Shi, Zhiguo},
  title = {Collaborative Sensing in Internet of Things: A Comprehensive Survey},
  year = {2022},
  issue_date = {thirdquarter 2022},
  publisher = {IEEE Press},
  volume = {24},
  number = {3},
  issn = {1553-877X},
  url = {https://doi.org/10.1109/COMST.2022.3187138},
  doi = {10.1109/COMST.2022.3187138},
  journal = {Commun. Surveys Tuts.},
  month = jul,
  pages = {1435--1474},
  numpages = {40}
}

@article{Siam2025AIoT,
  author = {Siam, Shakhrul Iman and Ahn, Hyunho and Liu, Li and Alam, Samiul and Shen, Hui and Cao, Zhichao and Shroff, Ness and Krishnamachari, Bhaskar and Srivastava, Mani and Zhang, Mi},
  title = {Artificial Intelligence of Things: A Survey},
  year = {2025},
  issue_date = {January 2025},
  publisher = {Association for Computing Machinery},
  address = {New York, NY, USA},
  volume = {21},
  number = {1},
  issn = {1550-4859},
  url = {https://doi.org/10.1145/3690639},
  doi = {10.1145/3690639},
  journal = {ACM Trans. Sen. Netw.},
  month = jan,
  articleno = {9},
  numpages = {75}
}

@ARTICLE{Li2023Heterogeneous,
  author = {{Li}, Zhize and {Liu}, Jun and {Chen}, Kezhou and {Gao}, Xiang and {Tang}, Chenshuo and {Xie}, Chao and {Lu}, Xu},
  title = "{Heterogeneous sensing for target tracking: architecture, techniques, applications and challenges}",
  journal = {Measurement Science and Technology},
  year = 2023,
  month = jul,
  volume = {34},
  number = {7},
  eid = {072002},
  pages = {072002},
  doi = {10.1088/1361-6501/acc267},
  adsurl = {https://ui.adsabs.harvard.edu/abs/2023MeScT..34g2002L}
}

@misc{Buyukkalayci2026TopP,
  title = {Top-P Sensor Selection for Target Localization},
  author = {Kaan Buyukkalayci and Kyle Pak and Merve Karakas and Xinlin Li and Christina Fragouli},
  year = {2026},
  eprint = {2604.07020},
  archivePrefix = {arXiv},
  primaryClass = {cs.IT},
  url = {https://arxiv.org/abs/2604.07020}
}

@article{Liu2022SensorSelection,
  author = {Liu, Changyi and Di, Kuangyu and Li, Tiancheng and Elvira, Victor},
  title = {A sensor selection approach to maneuvering target tracking based on trajectory function of time},
  journal = {EURASIP Journal on Advances in Signal Processing},
  year = {2022},
  volume = {2022},
  number = {1},
  pages = {72},
  doi = {10.1186/s13634-022-00903-1},
  url = {https://doi.org/10.1186/s13634-022-00903-1},
  issn = {1687-6180}
}

@inproceedings{Subramaniam2024TrackMDP,
  author = {Subramaniam, Adarsh M. and Gerogiannis, Argyrios and Hare, James Z. and Veeravalli, Venugopal V.},
  booktitle = {ICASSP 2025 - 2025 IEEE International Conference on Acoustics, Speech and Signal Processing (ICASSP)},
  title = {Track-MDP: Reinforcement Learning for Target Tracking with Controlled Sensing},
  year = {2025},
  pages = {1--5},
  doi = {10.1109/ICASSP49660.2025.10890122}
}

@Article{Jiang2023SensorManagement,
  AUTHOR = {Jiang, Xiaoxiao and Ma, Tianming and Jin, Jie and Jiang, Yujie},
  TITLE = {Sensor Management with Dynamic Clustering for Bearings-Only Multi-Target Tracking via Swarm Intelligence Optimization},
  JOURNAL = {Electronics},
  VOLUME = {12},
  YEAR = {2023},
  NUMBER = {16},
  ARTICLE-NUMBER = {3397},
  URL = {https://www.mdpi.com/2079-9292/12/16/3397},
  ISSN = {2079-9292},
  DOI = {10.3390/electronics12163397}
}

@article{Lauri2022POMDPRobotics,
  author       = {Mikko Lauri and
                  David Hsu and
                  Joni Pajarinen},
  title        = {Partially Observable Markov Decision Processes in Robotics: A Survey},
  journal      = {IEEE Trans. Robotics},
  volume       = {39},
  number       = {1},
  pages        = {21--40},
  year         = {2023},
  url          = {https://doi.org/10.1109/TRO.2022.3200138},
  doi          = {10.1109/TRO.2022.3200138},
  bibsource    = {dblp computer science bibliography, https://dblp.org}
}

@article{Singh2023EdgeAI,
title = {Edge AI: A survey},
journal = {Internet of Things and Cyber-Physical Systems},
volume = {3},
pages = {71-92},
year = {2023},
issn = {2667-3452},
doi = {https://doi.org/10.1016/j.iotcps.2023.02.004},
url = {https://www.sciencedirect.com/science/article/pii/S2667345223000196},
author = {Raghubir Singh and Sukhpal Singh Gill}
}

@InProceedings{bogunovic2021,
  title = 	 { Stochastic Linear Bandits Robust to Adversarial Attacks },
  author =       {Bogunovic, Ilija and Losalka, Arpan and Krause, Andreas and Scarlett, Jonathan},
  booktitle = 	 {Proceedings of The 24th International Conference on Artificial Intelligence and Statistics},
  pages = 	 {991--999},
  year = 	 {2021},
  editor = 	 {Banerjee, Arindam and Fukumizu, Kenji},
  volume = 	 {130},
  series = 	 {Proceedings of Machine Learning Research},
  month = 	 {13--15 Apr},
  publisher =    {PMLR},
  url = 	 {https://proceedings.mlr.press/v130/bogunovic21a.html}
}

@InProceedings{wu2020stochastic,
  title = 	 {Stochastic Linear Contextual Bandits with Diverse Contexts},
  author =       {Wu, Weiqiang and Yang, Jing and Shen, Cong},
  booktitle = 	 {Proceedings of the Twenty Third International Conference on Artificial Intelligence and Statistics},
  pages = 	 {2392--2401},
  year = 	 {2020},
  editor = 	 {Chiappa, Silvia and Calandra, Roberto},
  volume = 	 {108},
  series = 	 {Proceedings of Machine Learning Research},
  month = 	 {26--28 Aug},
  publisher =    {PMLR},
  url = 	 {https://proceedings.mlr.press/v108/wu20c.html}
}

@inproceedings{he2022nearly,
author = {He, Jiafan and Zhou, Dongruo and Zhang, Tong and Gu, Quanquan},
title = {Nearly optimal algorithms for linear contextual bandits with adversarial corruptions},
year = {2022},
isbn = {9781713871088},
publisher = {Curran Associates Inc.},
address = {Red Hook, NY, USA},
booktitle = {Proceedings of the 36th International Conference on Neural Information Processing Systems},
articleno = {2508},
numpages = {12},
location = {New Orleans, LA, USA},
series = {NIPS '22}
}

@INPROCEEDINGS{marlin2023iobt,
  author={Marlin, Benjamin M. and Suri, Niranjan and Fang, Shiwei and Srivastiva, Mani B. and Samplawski, Colin and Wang, Ziqi and Wigness, Maggie},
  booktitle={MILCOM 2023 - 2023 IEEE Military Communications Conference (MILCOM)}, 
  title={IoBT-MAX: a Multimodal Analytics eXperimentation Testbed for IoBT Research}, 
  year={2023},
  volume={},
  number={},
  pages={127-132},
  doi={10.1109/MILCOM58377.2023.10356347}}

@inproceedings{huang2013,
author = {Huang, Po-Sen and He, Xiaodong and Gao, Jianfeng and Deng, Li and Acero, Alex and Heck, Larry},
title = {Learning deep structured semantic models for web search using clickthrough data},
year = {2013},
isbn = {9781450322638},
publisher = {Association for Computing Machinery},
address = {New York, NY, USA},
url = {https://doi.org/10.1145/2505515.2505665},
doi = {10.1145/2505515.2505665},
booktitle = {Proceedings of the 22nd ACM International Conference on Information \& Knowledge Management},
pages = {2333–2338},
numpages = {6},
location = {San Francisco, California, USA},
series = {CIKM '13}
}
